\documentclass{article}

\usepackage{iclr2026_conference,times}

\usepackage[most]{tcolorbox}
\usepackage[utf8]{inputenc} 
\usepackage[T1]{fontenc}    

\usepackage{url}            
\usepackage{booktabs}       
\usepackage{amsfonts}       
\usepackage{nicefrac}       
\usepackage{microtype}      
\usepackage{xcolor}         
\usepackage{wrapfig}
\usepackage{graphicx}

\title{Expert Merging in Sparse Mixture of Experts with Nash Bargaining}

\usepackage{thm-restate}
\usepackage{caption}
\usepackage{subcaption}
\usepackage{listings}
\usepackage{color}
\usepackage{amssymb}

\usepackage{amsthm,thmtools}
\usepackage{wrapfig}
\usepackage{tikz}
\usepackage{multicol, multirow}
\usetikzlibrary{positioning}
\RequirePackage{algorithm}
\usepackage{algorithmic}
\usepackage{booktabs}       
\usepackage{amsfonts}       
\usepackage{nicefrac}       
\usepackage{microtype}      
\usepackage{mathtools}
\usepackage{algorithm}                 

\definecolor{myblue}{rgb}{0.1,0.2,0.75}
\definecolor{lightblue}{HTML}{A6E5FF}

\usepackage{titletoc}
\newcommand\DoToC{%
  \startcontents
  \printcontents{}{1}{\textbf{Table of Contents}\vskip3pt\hrule\vskip5pt}
  \vskip3pt\hrule\vskip5pt
}

\newcommand{\tmn}[1]{{ #1}}

\theoremstyle{plain}
\newtheorem{thm}{Theorem}[section]

\newtheorem{axiom}[thm]{Axiom}
\declaretheorem[name=Theorem,numberwithin=section]{thm}
\declaretheorem[name=Lemma,numberwithin=section]{resultRestate}

\theoremstyle{definition}
\newtheorem{definition}[thm]{Definition}
\newtheorem{assumption}[thm]{Assumption}
\theoremstyle{remark}

\usepackage{microtype}
\usepackage{multirow}
\usepackage{pifont}
\usepackage{array}
\usepackage{enumitem}
\usepackage{hyperref}
\hypersetup{colorlinks=true, urlcolor=blue, linkcolor=blue, citecolor=blue}

\newcommand*{\B}[1]{\mathbf{#1}}

\newcommand{\suma}{\boldsymbol{\alpha}}

\newcommand{\bothParam}{\mathbf{E}}

\newcommand{\gradSymbol}{\Delta\boldsymbol{\mathcal{E}}}
\newcommand{\bothGrad}{\gradSymbol}

\newcommand{\momCoeff}{\beta}
\newcommand{\momVal}{\boldsymbol{\mu}}
\newcommand{\lr}{\gamma}

\newcommand{\spectralRadius}{\rho}
\newcommand{\identity}{\boldsymbol{I}}

\newcommand{\dynamics}{\boldsymbol{R}}
\newcommand{\dimSymbol}{d}

\newcommand{\bothDim}{\dimSymbol}
\newcommand{\eigval}{\lambda}
\usepackage{framed} 
\definecolor{shadecolor}{named}{lightgray}
\usepackage{enumitem}
\usepackage[font=small]{caption}
\newcommand{\namex}{\textsc{NAMEx}}
\newcommand{\camex}{\textsc{CAMEx}}
\newcommand{\epcamex}{\textsc{EP\,-\,CAMEx}}
\newcommand{\smoe}{\textsc{SMoE}}
\newcommand{\deept}{DeepSeek\,-\,MoE}
\newcommand{\qwen}{Qwen1.5\,-\,MoE}

\usepackage{microtype}
\newcolumntype{C}[1]{>{\centering\arraybackslash}m{#1}}
\usepackage{hyperref}       
\usepackage[capitalize,noabbrev]{cleveref}
\hypersetup{
    colorlinks=true,
    linkcolor=blue,
    filecolor=magenta,      
    urlcolor=cyan,
    citecolor=myblue
}

\usepackage[table]{xcolor} 
\usepackage{tikz}
\usepackage{pgfplots}
\usepackage{nicematrix}

\usepackage{soul}
\usepackage[subtle]{savetrees}
\usepackage[normalem]{ulem}
\usepackage{makecell}
\usepackage{titlesec}
\titlespacing\section{0pt}{0pt plus 0pt minus 2pt}{0pt plus 0pt minus 2pt}
\titlespacing\subsection{0pt}{0pt plus 0pt minus 2pt}{0pt plus 0pt minus 2pt}
\titlespacing\subsubsection{0pt}{0pt plus 0pt minus 2pt}{0pt plus 0pt minus 2pt}

\definecolor{gray}{rgb}{0.5,0.5,0.5}
\definecolor{pygreen}{rgb}{0.0, 0.5, 0.0}
\definecolor{pyred}{rgb}{0.7, 0.0, 0.0}
\definecolor{pyblue}{rgb}{0.0, 0.0, 0.7}
\definecolor{pygray}{rgb}{0.5, 0.5, 0.5}
\definecolor{pydarkgray}{rgb}{0.3, 0.3, 0.3}

\usepackage{url}
\usepackage{siunitx}
\usepackage{titlesec}
\titlespacing\section{0pt}{0pt plus 0pt minus 2pt}{0pt plus 0pt minus 2pt}
\titlespacing\subsection{0pt}{0pt plus 0pt minus 2pt}{0pt plus 0pt minus 2pt}
\titlespacing\subsubsection{0pt}{0pt plus 0pt minus 2pt}{0pt plus 0pt minus 2pt}

\usepackage[subtle]{savetrees}
\usepackage[normalem]{ulem}
\AtBeginDocument{%
  \addtolength\abovedisplayskip{-0.2\baselineskip}%
  \addtolength\belowdisplayskip{-0.2\baselineskip}%
 \addtolength\abovedisplayshortskip{-0.2\baselineskip}%
 \addtolength\belowdisplayshortskip{-0.2\baselineskip}%
}

\newcommand\blfootnote[1]{%
  \begingroup
  \renewcommand\thefootnote{}\footnote{#1}%
  \addtocounter{footnote}{-1}%
  \endgroup
}

\usepackage[font=small]{caption}
\author{Dung V. Nguyen\textsuperscript{1}$^{\ast}$ \quad  Anh T. Nguyen\textsuperscript{2}$^{\ast}$ \quad
  Minh H. Nguyen\textsuperscript{3} \quad
  Luc Q. Nguyen\textsuperscript{2} \quad
  Shiqi Jiang\textsuperscript{1} \\
  \textbf{Ethan Fetaya\textsuperscript{4} \quad  Linh Duy Tran\textsuperscript{5}$^{\dagger}$ \quad
  Gal Chechik\textsuperscript{4}$^{\dagger}$ \quad  Tan M. Nguyen\textsuperscript{1}$^{\dagger}$}
 \\
  \textsuperscript{1}Department of Mathematics, National University of Singapore \\
  \textsuperscript{2}Viettel AI, Viettel Group \\
  \textsuperscript{3}Faculty of Mathematics and Informatics, Hanoi University of Science and Technology \\
  \textsuperscript{4}Bar Ilan University, Israel\\
  \textsuperscript{5}AI Imaging Team, Data Solution Department, FPT Software Japan\\
   \texttt{\{dungnv, shiqijiang\}@u.nus.edu}, \texttt{tanmn@nus.edu.sg} \\
   \texttt{\{ethan.fetaya, gal.chechik\}@biu.ac.il} \\
   \texttt{ minh.nh232331M@sis.hust.edu.vn} \\
   \texttt{\{anhnt21, lucnq1\}@viettel.com.vn}\\
   \texttt{linhtd32@fpt.com}\\ 
}
\iclrfinalcopy
\begin{document}

\maketitle

\begin{abstract}
Existing expert merging strategies for Sparse Mixture of Experts (SMoE) typically rely on input-dependent or input-independent averaging of expert parameters, but often lack a principled weighting mechanism. In this work, we reinterpret expert merging through the lens of game theory, revealing cooperative and competitive dynamics among experts. Based on this perspective, we introduce Nash Merging of Experts (NAMEx), a novel framework that incorporates Nash Bargaining into the merging process, enabling more balanced and efficient collaboration among experts. Additionally, we incorporate complex momentum into NAMEx to accelerate expert propagation with theoretical guarantees for convergence. Extensive experiments across language modeling, text classification, image classification, and zero-shot robustness under data corruption show that NAMEx consistently outperforms competing methods while integrating seamlessly with popular MoE architectures. Finally, we demonstrate NAMEx’s scalability by applying it to large-scale systems, including Qwen1.5-MoE (14B) and DeepSeek-MoE (16B), where it proves effective in both zero-shot and fine-tuning settings. The code is publicly available at: 
\url{https://github.com/anh147/NAMEx}.
\blfootnote{$^{\ast}$Co-first authors. $^{\dagger}$Co-last authors. Correspondence to: dungnv@u.nus.edu \& tanmn@nus.edu.sg}
\end{abstract}

\section{Introduction}
Scaling up neural networks without proportional increases in computational cost is a key goal in modern deep learning. Sparse Mixture of Experts (SMoE) architectures offer a powerful solution: they selectively activate only a subset of expert modules for each input, thereby maintaining high capacity while preserving computational efficiency. 
Building on the classical Mixture of Experts (MoE) framework \citep{jacobs1991adaptive}, SMoE leverages a dynamic gating mechanism to determine which experts participate in processing a given input. This sparsity allows extremely large models to be trained efficiently and has shown promise across natural language processing \citep{shazeer2017outrageously, liu2024deepseek, qwen2}, and computer vision \citep{riquelme2021scaling, puigcerver2023softmoe, nielsen2025tight} applications.

A core component of SMoE is the routing mechanism, which dynamically determines expert assignments. Significant efforts have focused on improving routing stability, load balancing, and expressiveness. For example, StableMoE \citep{dai2022stable} introduces a two-stage strategy to reduce routing variance; SMEAR \citep{muqeeth2023soft} proposes soft parameter merging via weighted averaging to bypass discrete selection; and HyperRouter \citep{do2023hyperrouter} uses hypernetworks to generate router parameters. Meanwhile, SoftMoE \citep{puigcerver2023softmoe} blends sparse and dense routing, and patch-level routing \citep{chowdhury2023patch} improves sample efficiency in visual tasks. Beyond these, Switch Transformer \citep{fedus2022switch} simplifies routing to top-$1$ expert selection with auxiliary load-balancing losses for stable large-scale training, GShard \citep{lepikhin2020gshard} introduces expert capacity constraints and scalable sharding strategies, and BASE Layers \citep{lewis2021base} formulates routing as a balanced assignment problem to improve expert utilization. Hash-based routing \citep{roller2021hash} replaces learned routers with deterministic hashing to reduce overhead, while DSelect-k \citep{hazimeh2021dselect} provides a differentiable relaxation for sparse expert selection to improve gradient flow. Collectively, these approaches span hard, sparse routing with explicit balancing objectives to soft and differentiable formulations, reflecting a broad design space aimed at enhancing stability, scalability, and specialization in SMoE systems.
  \begin{wrapfigure}[23]{r}{0.4\linewidth}
    \centering
    \vspace{-1em}
    \includegraphics[width=\linewidth]{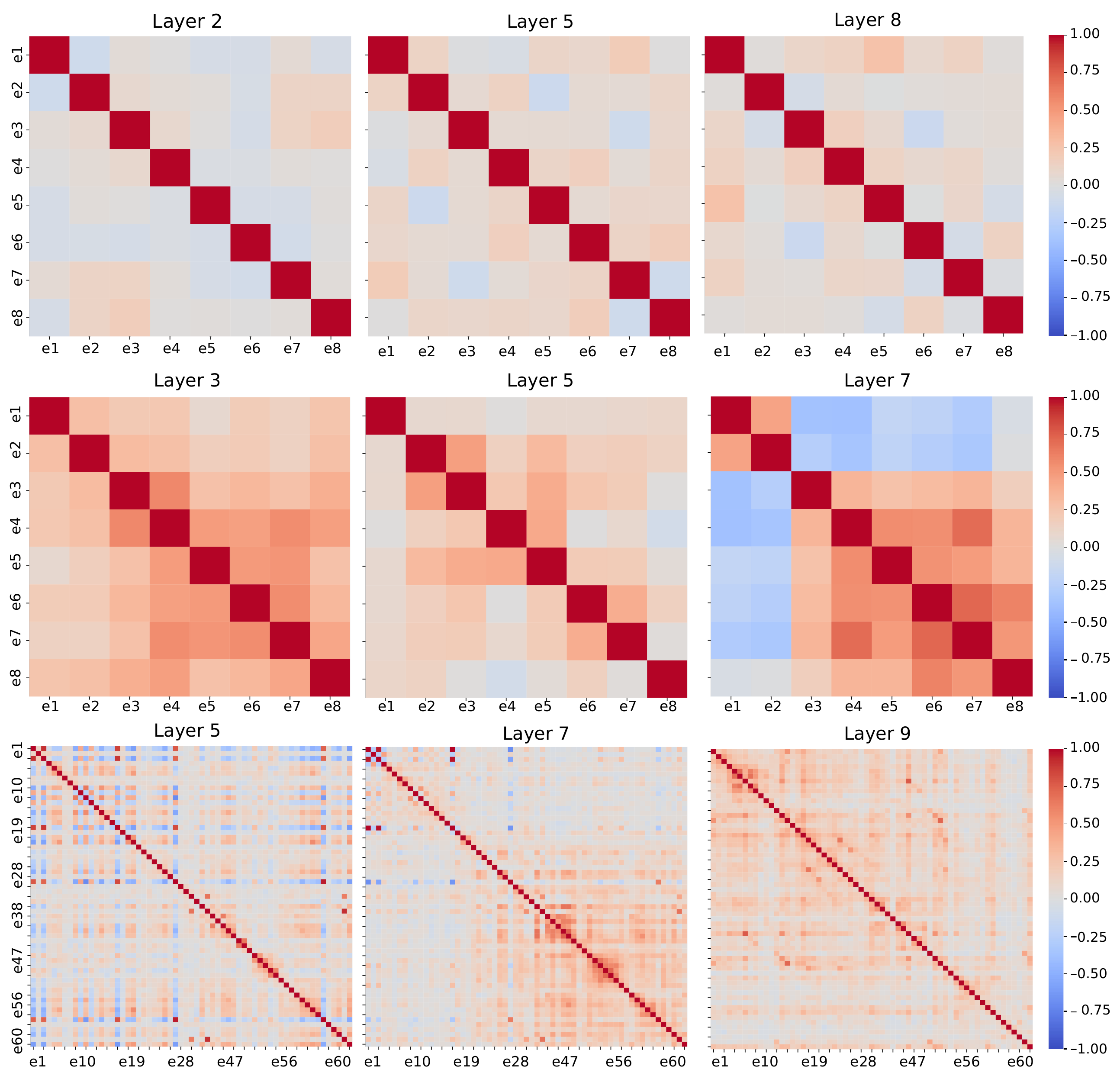}
    \caption{Cosine similarity of expert outputs in Swin-MoE~\citep{liu2021Swin} (top), Switch-Transformer~\citep{fedus2022switch} (middle), and Qwen-MoE~\citep{qwen2} (bottom). Swin-MoE shows stable mid-layer features, Switch-Transformer exhibits dynamic routing at Layer 8, and Qwen-MoE yields robust final representations at Layer 9--highlighting diverse expert interaction patterns.}
    \label{fig:experts-heat}
\end{wrapfigure}
Beyond routing, a complementary yet underexplored direction is \textit{expert merging}. Instead of selecting a subset of experts per input, merging aims to combine all expert parameters into a unified model, either during training or at inference. This approach is especially appealing when deployment or memory constraints demand a single-expert representation. Merging is particularly valuable in autoregressive models \citep{zhong2024lory} and cross-domain transfer settings \citep{chen2022towards}. However, most current merging techniques, such as soft-merging \citep{muqeeth2023soft, zhong2024lory} and top-$k$ aggregation \citep{he-etal-2023-merging, li2024merge}, rely on heuristic weighting schemes that ignore the intricate dynamics between experts.

Recent work has begun to address this limitation.  \citep{iclr2024camex} introduces a curvature-aware merging scheme, namely Curvature-aware merging of experts (CAMEx), that uses natural gradients to account for non-Euclidean geometry in parameter space. A variant, corresponding to the dynamic merging (Dynamic-Merg) mechanism in the CAMEx paper, which we refer to as Expert-Propagation CAMEx (EP-CAMEx), propagates a base expert across layers to promote inter-layer communication. However, despite its elegance, EP-CAMEx underperforms its static variant, likely due to insufficient coordination among expert contributions. 
This motivates a deeper question: \textit{Can we interpret expert merging as a structured interaction among experts, rather than just a linear average?}\\

{\bf Contribution.} In this paper, we frame expert merging as a cooperative-competitive game among experts. Drawing inspiration from multi-task learning, we adopt the \textit{Nash Bargaining Solution} (NBS) \citep{nash1950barganing} to derive merging coefficients from first principles based on each expert’s contribution. Our method, named \emph{Nash Merging of Experts} (NAMEx), treats expert domain vectors as utility functions in a bargaining game. By solving for the optimal agreement point, NAMEx ensures a fair and efficient merging process that reflects expert alignment and divergence.

To address the slow convergence of EP-CAMEx, we further integrate \emph{complex momentum} \citep{lorraine2022complex} 
into the propagation process. This enhancement accelerates convergence while preserving stability, especially when expert interactions include adversarial or conflicting dynamics. 
We theoretically prove the convergence of NAMEx under mild conditions and provide a spectral radius-based bound for the convergence rate of NAMEx-Momentum. Our contribution is three-fold:
 \begin{enumerate}[leftmargin=24pt]
     \item We develop NAMEx, a new expert merging method that integrates the Nash Bargaining optimization framework of \citep{navon2022multi} into EP-CAMEx~\citep{iclr2024camex}, improving expert propagation at each SMoE layer.
     \item We incorporate complex momentum into our NAMEx to enhance the stability and convergence speed of expert propagation across layers and provide theoretical guarantees.
     \item We demonstrate that quaternion momentum presents a promising future direction for further improving expert merging. 
 \end{enumerate}
Comprehensive experiments across diverse tasks--including WikiText-103 language modeling~\citep{merity2016pointer}, GLUE text classification finetuning~\citep{wang2019glue}, and ImageNet-1k image classification and zero-shot robustness under data corruption~\citep{deng2009imagenet}--demonstrate the effectiveness of our approach, achieving superior accuracy compared to baseline methods while preserving advantages in computational efficiency. Moreover, we establish NAMEx's scalability by deploying it on large systems such as Qwen1.5-MoE (14B) and DeepSeek-MoE (16B), where it delivers strong performance in both zero-shot and fine-tuning scenarios.

 
 {\bf Organization.} Section~\ref{sec:background} reviews SMoE, CAMEx, and Nash Bargaining; Section~\ref{sec:namex} introduces NAMEx and its momentum extension; Section~\ref{sec:experiment} presents experiments with ablations; Section~\ref{sec:related} discusses related work; and Section~\ref{sec:conclusion} concludes with limitations.
 

\section{Background}
\label{sec:background}
\subsection{Sparse Mixture of Experts}
The Mixture of Experts (MoE) framework enables modular neural computation by combining multiple specialized sub-networks (\textit{experts}) through a gating function~\citep{jacobs1991adaptive}. The Sparse Mixture of Experts (SMoE) variant enhances scalability by activating only a small subset of experts per input, significantly reducing computation during training and inference~\citep{shazeer2017outrageously, lepikhin2021gshard,fedus2022switch}. 

Let $\mathbf{x} \in \mathbb{R}^d$ be an input and ${f_i(\mathbf{x})}_{i=1}^N$ denote expert outputs. A gating network computes weights $s_i(\mathbf{x})$ such that:
\begin{equation}
s_i(\mathbf{x}; \mathbf{\theta}_g) \geq 0, \quad \sum_{i=1}^N s_i(\mathbf{x}; \mathbf{\theta}_g) = 1, \quad F(\mathbf{x}) = \sum_{i=1}^N s_i(\mathbf{x}; \mathbf{\theta}_g) f_i(\mathbf{x}).
\end{equation}
SMoE improves model capacity without linearly scaling compute, making it a central design in recent large-scale architectures.
\begin{figure*}[t!]
    \centering
    \includegraphics[width=0.8\linewidth]{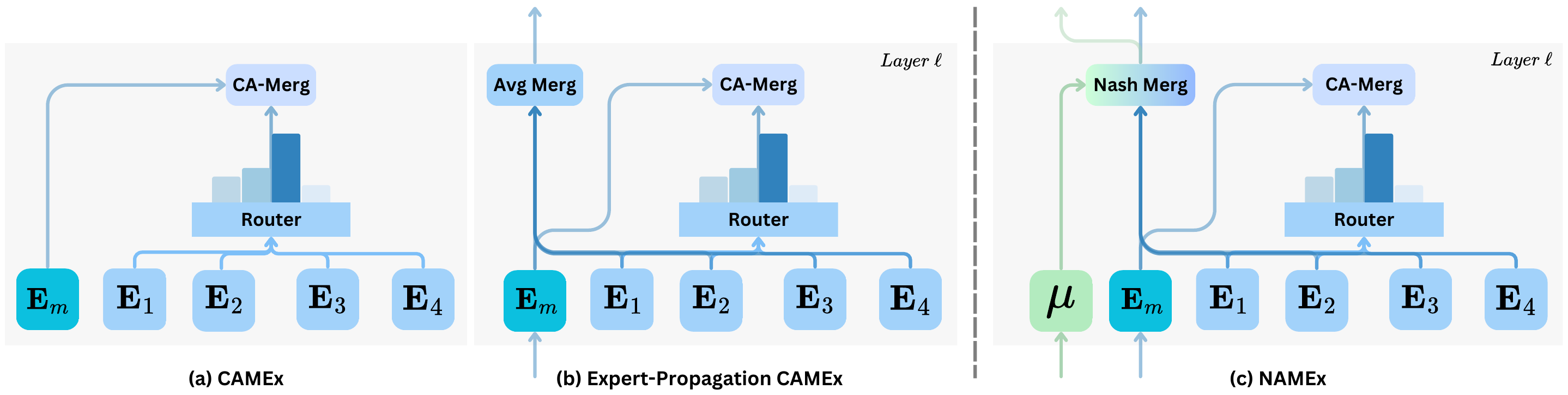}
    \caption{Architecture overview of (a) CAMEx \citep{iclr2024camex}, (b) Expert-Propagation CAMEx \citep{iclr2024camex}, and (c) our proposed merging method, NAMEx.}
    \label{fig:architect}
\vspace{-0.2in}
\end{figure*}

\subsection{Curvature-aware Merging of Experts}
CAMEx uses natural gradients to align merged experts more closely with the geometry of the parameter space, enhancing both pre-training and fine-tuning processes~\citep{iclr2024camex}. Hence, CAMEx generalizes popular expert merging methods such as SMEAR \citep{muqeeth2023soft} and Lory \citep{zhong2024lory} and can be formulated as the following natural gradient-like merging scheme:
\begin{equation}
\label{eq:domain-specific}
    \mathbf{\hat{E}}^{(l)}_m = \mathbf{E}^{(l)}_m + \eta \sum_{i=1}^{N}  \B{M}^{(l)}_i \cdot (s_i^{(l)} *\tau_i^{(l)}),
\end{equation}
where $\tau^{(l)}_i = \mathbf{E}^{(l)}_i - \mathbf{E}^{(l)}_m$ is the domain-vector of the $i$-th expert, representing its deviation from the base expert. $\mathbf{E}^{(l)}_i$, $\mathbf{E}^{(l)}_m$, and $\mathbf{\hat{E}}^{(l)}_m$ denote the weights of the $i$-th expert, the base expert, and the resulting merged expert that processes the input, respectively. Here, the base expert $\mathbf{E}^{(l)}_m$ is shared between tokens in layer $l$, just like in DeepSeek-V2~\citep{liu2024deepseek} and V3~\citep{liu2024deepseekv3}.  $\eta > 0$ denotes the stepsize for updating the base expert, and $\B{M}^{(l)}_i$ is the curvature matrix for the $i$-th expert. 
 EP-CAMEx is an extension of CAMEx, in which the base expert $\bothParam^{(0)}_m$ is initialized at the first layer, and $\bothParam^{(l)}_m$ and $\hat\bothParam^{(l)}_m$ are updated at subsequent layers as follows:
\begin{equation}
 \begin{cases}
 \label{eq:camex-dynamic}
    \mathbf{E}_m^{(l+1)} &= \mathbf{E}_m^{(l)} + \dfrac{\gamma}{N} \displaystyle\sum_{i=1}^{N} \mathbf{M}^{(l)}_i \cdot \tau_i^{(l)}, \\
    \mathbf{\hat{E}}_m^{(l+1)} &= \mathbf{E}_m^{(l+1)} + \eta \displaystyle\sum_{i=1}^{N} \mathbf{M}^{(l+1)}_i \cdot (s^{(l+1)}_i * \tau_i^{(l+1)}).
\end{cases}   
\end{equation}
Here, $\gamma > 0$ denotes the step size for the propagation of the base expert $\bothParam^{(l)}_m$. In the first equation of system \eqref{eq:camex-dynamic} above, if we view each domain-vector $\tau_i^{(l)}$ as a ``gradient direction'' attempting to pull the base expert toward the corresponding expert's domain, then the formulation can be interpreted as a dynamical system that updates $\mathbf{E}^{(l)}_m$ using Multiple-Gradient Descent Algorithm (MGDA) \citep{desideri2012mdga} to minimize the distance between $\mathbf{E}_m$ and $i$-th domain. Consequently, this can be framed as a multi-objective optimization or multi-task learning problem. We illustrate both CAMEx and EP-CAMEx in Figure~\ref{fig:architect}(a) and (b).

\subsection{Nash Bargaining in Multi-Task Learning}

The Nash Bargaining Solution (NBS) \citep{nash1950barganing} is a foundational concept in cooperative game theory, describing how multiple agents can reach a fair and Pareto-optimal agreement. A bargaining problem is typically defined by a agreement set of outcomes $\mathcal{S} \subseteq \mathbb{R}^N$ and a disagreement point $\mathbf{d} \in \mathbb{R}^N$, which specifies the utility each player receives if no agreement is reached. The NBS selects an outcome $\mathbf{u}^* \in \mathcal{S}$ that maximizes the product of individual gains over the disagreement point:
\begin{equation}
    \mathbf{u}^* = \arg\max_{\mathbf{u} \in \mathcal{S}} \prod_{i=1}^N (u_i - d_i).
\end{equation}
The disagreement point in the Nash Bargaining Problem is the fallback outcome each player receives if no agreement is reached. It serves as a baseline against which any cooperative agreement is measured, shaping the set of feasible solutions. Players often consider their disagreement point strategically, as improvements to it can strengthen their bargaining position and influence the final outcome. 

Recent work by~\citep{navon2022multi} demonstrates that multi-task learning (MTL) can be naturally framed as a bargaining game. In this setting, each task corresponds to a player, and the goal is to determine a shared parameter update direction $\Delta \boldsymbol{\theta}$ that benefits all tasks. The agreement set is typically constrained to a unit ball $B_\epsilon = \{ \Delta \boldsymbol{\theta} \mid \|\Delta \boldsymbol{\theta}\| \leq \epsilon \}$, while the disagreement point is set to zero, indicating no parameter update. Each task $i$ provides a utility function
\begin{equation}
\label{eq:utility-mtl}
u_i(\Delta \boldsymbol{\theta}) = \boldsymbol{\tau}_i^\top \Delta \boldsymbol{\theta},
\end{equation}
where $\boldsymbol{\tau}_i$ is the gradient of the task-specific loss with respect to the model parameters. Under the assumption that these gradients are linearly independent, the NBS yields the optimal update direction:
\begin{equation}
\Delta \boldsymbol{\theta} = \sum_{i=1}^N \alpha_i \boldsymbol{\tau}_i, \quad \text{where } \B{G}^\top \B{G} \boldsymbol{\alpha} = 1 / \boldsymbol{\alpha},
\end{equation}
with $\B{G} = [\boldsymbol{\tau}_1, \dots, \boldsymbol{\tau}_N]$ and $1 / \boldsymbol{\alpha}$ denoting element-wise reciprocals.

This formulation provides a principled way to resolve conflicting gradients, balancing cooperative and adversarial dynamics among tasks. In this paper, we leverage this framework to reinterpret expert merging in SMoE and particularly, CAMEx, as a bargaining game among experts, where each domain vector $\boldsymbol{\tau}^{(l)}_i$, $i=1,\dots,N$, plays the role of a task gradient.

\section{Nash Merging of Experts}
\label{sec:namex}
Building on the foundations of CAMEx and Nash Bargaining, we now introduce NAMEx--a novel method for merging experts in SMoE via Nash Bargaining. Rather than treating expert merging as a simple averaging task, NAMEx models it as a multi-agent bargaining game, where each expert proposes a directional update, i.e., its domain vector, and the merged expert is obtained through a principled aggregation reflecting both cooperation and competition. To address slow convergence in existing propagation methods like EP-CAMEx, we further introduce \emph{complex momentum} into NAMEx, enabling faster and more stable propagation through SMoE layers. An overview of our approach is shown in \cref{fig:architect}(c).

\subsection{Merging Experts as a Bargaining Game}
Setting $\bothGrad^{(l)}$ as an update direction for $\mathbf{E}^{(l)}_m$ of the $l$-th layer in the first equation of system \eqref{eq:camex-dynamic}, we adjust the expert-propagating updating step in EP-CAMEx as follows:
\begin{equation}
 \begin{cases}
 \label{eq:dynamic}
    \mathbf{E}_m^{(l+1)} &= \mathbf{E}_m^{(l)} + \gamma \bothGrad^{(l)}, \\
    \mathbf{\hat{E}}_m^{(l+1)} &= \mathbf{E}_m^{(l+1)} + \eta \displaystyle\sum_{i=1}^{N} \mathbf{M}_i \cdot (s^{(l+1)}_i * \boldsymbol{\tau}_i^{(l+1)}).
\end{cases}   
\end{equation}
Like CAMEx, we view the domain-vectors $\boldsymbol{\tau}_i^{(l)} = \mathbf{E}^{(l)}_i - \mathbf{E}^{(l)}_m$ as analogous to a gradient step that pulls $\mathbf{E}_m$ toward $\mathbf{E}_i$'s domain. However, different from the formulation of EP-CAMEx in system~\eqref{eq:camex-dynamic}, we remove the curvature matrix in the first equation to align with the Bargaining Game given by Algorithm 1 in \citep{navon2022multi}. Our goal now is to find an optimal update vector $\bothGrad^{(l)}$ which benefits all experts, i.e.,  finding $\boldsymbol{\alpha}^{(l)} = [\alpha^{(l)}_1, \alpha^{(l)}_2, \dots, \alpha^{(l)}_N]$ to aggregate the domain-vectors $\boldsymbol{\tau}_i^{(l)}$ into $\bothGrad^{(l)}$ as in Eqn.~\ref{eq:dynamic}. We hypothesize that experts in SMoE engage in \emph{mixed games} comprising both cooperative and competitive dynamics. 

{\bf Layer-Wise Expert Interaction Dynamics.} Following the analysis protocol described in \citep{lo2025closer}, we observe that expert behavior varies by layer and architecture (see \cref{fig:experts-heat}), revealing both cooperative and adversarial patterns. For instance, in Swin-MoE \citep{liu2021Swin}, middle layers show high inter-expert similarity, while Qwen-MoE \citep{qwen2} concentrates alignment in deeper layers. This motivates a dynamic, layer-wise approach to merging--exactly what NAMEx provides. Please refer to Figure \ref{fig:swin-heatmap} and Figure \ref{fig:switch-heatmap} in Appendix~\ref{app:extra_emp_analysis} for more analysis on the dynamic of expert interaction. For comparison regarding expert interaction patterns under the impact of Load Balancing loss, please refer to Figure \ref{fig:lb-and-nlb} in Appendix \ref{app:extra_emp_analysis}.

Adapting the bargaining game's formulation in \citep{navon2022multi}, NAMEx solves the following problem: 


\begin{tcolorbox}[
  colback=gray!5,      
  colframe=black!70,   
  boxrule=0.8pt,       
  arc=2mm,             
  enhanced,
  width=\textwidth,
  boxsep=4pt,          
  left=6pt, right=6pt, top=4pt, bottom=4pt,
]
[{\bf Bargaining of Expert Merging (BEM) Problem}] \emph{Given an experts-merging problem with the set of expert parameters $\{\mathbf{E}_1, \mathbf{E}_2, \dots, \mathbf{E}_N\}$  and the base expert's parameter  $\mathbf{E}_m $, find an update vector $\bothGrad$ within a ball $B_\epsilon$ of radius $\epsilon$ centered at zero, i.e., $B_\epsilon = \{ \bothGrad \mid \|\bothGrad\| \leq \epsilon \}$.} 
\end{tcolorbox}

Inspired by \citep{navon2022multi}, in this bargaining problem, we set the disagreement point to $0$, corresponding to not updating $\mathbf{E}_m$. Similar to Eqn.~\ref{eq:utility-mtl}, the utility function for each expert is defined as $u_i(\bothGrad) = \boldsymbol{\tau}_i^\top \bothGrad$, where $\boldsymbol{\tau}_i$ is the domain-vector for expert $i$, representing its deviation from the base expert and capturing its unique contribution to the merging process. Here, $\bothGrad$ is equivalent to $\Delta \boldsymbol{\theta}$ in Eqn.~\ref{eq:utility-mtl}. We have the following mild axiom on the Nash bargaining solution.
\begin{axiom}[Pareto optimality of Nash bargaining solution \citep{nash1950barganing}]
\label{ass:solvable}
The selected agreement must be Pareto efficient, i.e. no other feasible outcome should exist that improves one player's utility without reducing the utility of at least one other player.
\end{axiom}
Under Axiom~\ref{ass:solvable}, the solution to the BEM Problem above is given by the following lemma.
\begin{restatable}[Nash Solution of Expert Merging]{lemma}{Resutlnash}
\label{lem:alphas}
Let $\B{G}$ denote the $d \times N$ matrix whose columns are the domain-vectors $\tau_i$. The solution to 
\begin{equation}
\arg\max_{\bothGrad \in B_\epsilon} \sum_{i=1}^N \log(\bothGrad^\top \tau_i)
\end{equation}
is (up to scaling) $\bothGrad^* = \sum_{i=1}^N \alpha_i \tau_i$, where $\alpha \in \mathbb{R}^N_+$ satisfies $\B{G}^\top \B{G} \alpha = 1 / \alpha$, with $1 / \alpha$ being the element-wise reciprocal.
\end{restatable}
Note that, under Axiom~\ref{ass:solvable}, it can be proven that the Nash solution to the bargaining problem is not dominated by other solutions. A proof sketch for Lemma \ref{lem:alphas} is provided in Appendix \ref{pf:lm3.1}.


\subsection{NAMEx as the Nash Solution of Expert Merging}
\label{sec:namex-def}
We now formally define NAMEx as the Nash Bargaining Solution to the BEM problem.

\begin{definition}[NAMEx: Nash Merging of Experts]
\label{def:namex}
{\it Let $\{\mathbf{E}_1^{(l)}, \dots, \mathbf{E}_N^{(l)}\}$ be the expert parameters and let $\mathbf{E}_m^{(l)}$ denote the base expert at layer $l$. Define the domain-vectors as $\boldsymbol{\tau}_i^{(l)} = \mathbf{E}_i^{(l)} - \mathbf{E}_m^{(l)}$, and let $\B{G}^{(l)} = [\boldsymbol{\tau}_1^{(l)}, \dots, \boldsymbol{\tau}_N^{(l)}]$ be the matrix formed by stacking these vectors. The \textit{NAMEx update direction} $\bothGrad^{(l)}$ is defined as:
$$
\bothGrad^{(l)} = \sum_{i=1}^N \alpha_i^{(l)} \boldsymbol{\tau}_i^{(l)},
$$
where $\boldsymbol{\alpha}^{(l)} \in \mathbb{R}_+^N$ satisfies the Nash Bargaining equation:$ {\B{G}^{(l)}}^\top \B{G}^{(l)} \boldsymbol{\alpha}^{(l)} = 1 / \boldsymbol{\alpha}^{(l)}$, with $1 / \boldsymbol{\alpha}^{(l)}$ denoting the element-wise reciprocal. The NAMEx update then proceeds by plugging \textit{NAMEx update direction} into Eqn.~\ref{eq:dynamic}:
\begin{equation}
 \begin{cases}
 \label{eq:namex}
    \mathbf{E}_m^{(l+1)} &= \mathbf{E}_m^{(l)} + \gamma \displaystyle\sum_{i=1}^N \alpha_i^{(l)} \boldsymbol{\tau}_i^{(l)}, \\
    \mathbf{\hat{E}}_m^{(l+1)} &= \mathbf{E}_m^{(l+1)} + \eta \displaystyle\sum_{i=1}^{N} \mathbf{M}_i \cdot (s^{(l+1)}_i * \boldsymbol{\tau}_i^{(l+1)}),
\end{cases}   
\end{equation}
where $\gamma, \eta \in \mathbb{R}_+$ are step-size coefficients, $\mathbf{M}_i$ is the curvature matrix for expert $i$, and $s_i^{(l+1)}$ are the routing weights at layer $l+1$.}
\end{definition}

We summarize the implementation of NAMEx in Algorithm \ref{alg:namex}.

{\bf Dissecting NAMEx.} We now discuss the behavior of NAMEx by studying the Nash Solution of the BEM Problem. First, if all $\tau_j$ are orthogonal, we obtain $\alpha_j = \dfrac{1}{\|\tau_j\|}$ and $\bothGrad = \sum_{j=1}^N \alpha_j \tau_j $, which is a scale-invariant solution. When $\tau_j$ are not orthogonal, we obtain
\begin{equation}
    \alpha_j \| \tau_j \|^2 + \sum_{i \neq j} \alpha_i \tau_i^\top \tau_j = \frac{1}{\alpha_j}. 
    \label{eq:alpha_sol}
\end{equation}
Lemma \ref{lem:alphas} allows us to calculate the optimal update direction $\bothGrad$ for an expert-propagation step at $l$-th layer as $\bothGrad^{(l)} = \sum_{i=1}^N \alpha_i \tau_i^{(l)}.$

Furthermore, assuming that EP-CAMEx obeys the update law in Eqn. 7 in \citep{desideri2012mdga}, the norm $\|\tau_j\|$ is (nearly) identical between domain vectors, we can view the expert update step in \citep{iclr2024camex} as a trivial solution (with a scaling factor) of Lemma \ref{lem:alphas}, ignoring the interaction between experts. While they also apply curvature matrices to the expert propagating step, the learned curvature matrices provide no additional information about other experts. Thus, the conclusion still holds.

\begin{center}
{\small
\begin{minipage}[t!]{.45\textwidth}
        \begin{algorithm}[H]
            \caption{Expert Merging via Nash Bargaining} \label{alg:expert_merging}
            \label{alg:namex}
            \begin{algorithmic}[1]
                \STATE \textbf{Initialize:} Model $M$ with $L$ SMoE layers, number of experts $N$, $\gamma, \eta \in \mathbb{R}^+$
                \STATE \hspace{0.5em} $H^{(t)} \in \mathbb{R}^{B \times S \times N}$: router logits at layer $t$
                \STATE \hspace{0.5em} $T^{(t)} \in \mathbb{R}^{B \times S \times D}$: token sequence at layer $t$
                \FOR{$t = 1$ to $L$}
                    \FOR{$i = 1$ to $N$}
                        \STATE $\tau^{(t)}_i \gets \B{E}^{(t)}_i - \B{E}^{(t)}_m$
                    \ENDFOR
                    \STATE $\B{G}^{(t)} \gets [\tau^{(t)}_1, \tau^{(t)}_2, \dots, \tau^{(t)}_N]$
                    \STATE Solve for $\alpha$: $\left(\B{G}^{(t)}\right)^\top \B{G}^{(t)} \boldsymbol{\alpha} = 1 / \boldsymbol{\alpha}$
                    \STATE $\B{E}^{(t+1)}_m \gets \B{E}^{(t)}_m + \gamma \sum_i \tau^{(t)}_i \alpha_i$
                    \STATE $\B{E}^{(t+1)}_m \gets \B{E}^{(t+1)}_m + \eta \sum_i H^{(t)}_i \cdot \tau^{(t)}_i$
                \ENDFOR
            \end{algorithmic}
        \end{algorithm}
\end{minipage}\hfil
\begin{minipage}[t!]{.45\textwidth}
        \begin{algorithm}[H]
            \caption{NAMEx-Momentum} \label{alg:simultaneous_complex}
\begin{algorithmic}[1]
    \STATE \textbf{Input:} $\gamma \in \mathbb{R}^+$, $\beta \in \mathbb{C}$, $\mu^{(0)} \in \mathbb{C}^{d}$, $\B{E}^{(0)} \in \mathbb{R}^{d}$
    \FOR{$j = 1$ to $L-1$}
        \FOR{$i = 1$ to $N$}
            \STATE $\boldsymbol{\tau}_i^{(j)} \gets \B{E}_i^{(j)} - \B{E}_m^{(j)}$
        \ENDFOR
        \STATE $\B{G}^{(j)} \gets [\boldsymbol{\tau}_1^{(j)}, \dots, \boldsymbol{\tau}_N^{(j)}]$
        \STATE Solve $\boldsymbol{\alpha}$ from: $(\B{G}^{(j)})^\top \B{G}^{(j)} \boldsymbol{\alpha} = 1 / \boldsymbol{\alpha}$
        \STATE $\bothGrad^{(j)} \gets \sum_i \boldsymbol{\tau}_i^{(j)} \alpha_i$ \hfill\COMMENT{\textit{Same update as NAMEx}}
        \STATE $\mu^{(j+1)} \gets \beta \mu^{(j)} + \bothGrad^{(j)}$
        \STATE $\B{E}^{(j+1)} \gets \B{E}^{(j)} + \Re(\gamma \mu^{(j+1)})$
        \STATE \textit{// Optional: Add residual alignment or router term if needed}
    \ENDFOR
\end{algorithmic}
        \end{algorithm}
\end{minipage}
}
\end{center}
In Eqn.~\ref{eq:alpha_sol}, we can consider $\sum_{i \neq j} \alpha_i \tau_i^\top \tau_j = (\sum_{i \neq j} \alpha_i \tau_i^\top) \tau_j$ as the interaction between the $j$-th expert and the other experts. If the sum is positive, the experts cooperate, and the other domain-vectors aid the $j$-th expert. $\alpha_j$ decreases in this case. If the sum is negative, the other experts hamper the $j$-th expert, i.e., an adversarial behavior between experts, and therefore, $\alpha_j$ increases to ensure that Eqn.~\ref{eq:alpha_sol} holds.

\subsection{Integrating Momentum Into Expert Merging}
We hypothesize that one reason for EP-CAMEx's inferior performance compared to CAMEx is its reliance on a fixed number of update steps, constrained by the model’s layer count. This limitation hinders the convergence of the base expert in later stages, leading to suboptimal performance. To mitigate this, we introduce momentum to accelerate convergence during optimization. In particular, we adopt complex momentum~\citep{lorraine2022complex}, which has been shown to be more robust and effective than standard first-order methods across a wide range of cooperative and adversarial games. By integrating complex momentum into expert merging, we enhance the propagation of expert updates across layers and provide theoretical support for its improved convergence rate. 

We present a formal definition NAMEx-Momentum below and  summarize an algorithm to implement it in \cref{alg:simultaneous_complex}.

\begin{definition}[NAMEx-Momentum: Nash Merging with Complex Momentum]
\label{def:namex-momentum}
{\it Let $\{\mathbf{E}_1, \dots, \mathbf{E}_N\}$ be the expert parameters and $\mathbf{E}_m$ the base expert at a given layer. Define the domain-vectors $\boldsymbol{\tau}_i = \mathbf{E}_i - \mathbf{E}_m$ and the matrix $\B{G} = [\boldsymbol{\tau}_1, \dots, \boldsymbol{\tau}_N]$. At each iteration $j$, the update direction $\bothGrad^{(j)} = \sum_{i=1}^N \alpha_i \boldsymbol{\tau}_i$ is computed where $\boldsymbol{\alpha}$ solves the Nash system:
\[
\B{G}^\top \B{G} \boldsymbol{\alpha} = 1 / \boldsymbol{\alpha}.
\]
\textit{NAMEx-Momentum} uses a complex momentum buffer $\mu^{(j)} \in \mathbb{C}^d$ to accumulate directional updates:
\begin{equation}
\begin{cases}
    \mu^{(j+1)} &= \beta \mu^{(j)} + \bothGrad^{(j)}, \\
    \mathbf{E}_m^{(j+1)} &= \mathbf{E}_m^{(j)} + \Re(\gamma \mu^{(j+1)})\\
    \mathbf{\hat{E}}_m^{(l+1)} &= \mathbf{E}_m^{(l+1)} + \eta \displaystyle\sum_{i=1}^{N} \mathbf{M}_i \cdot (s^{(l+1)}_i * \boldsymbol{\tau}_i^{(l+1)}),
\end{cases}
\end{equation}
where $\beta \in \mathbb{C}$ is the momentum coefficient, $\gamma \in \mathbb{R}^+$ is the step size, and $\Re(\cdot)$ denotes the real part.}
\end{definition}

We provide a convergence guarantee for NAMEx-Momentum update in Proposition~\ref{prop:theorem_existence} below and Theorem~\ref{theo:converge_complex} in Appendix~\ref{app:proofs}. Their proofs are in Appendix~\ref{pf:prop}
\begin{restatable}[Convergence rate of NAMEx-Momentum]{proposition}{resultExist}
    \label{prop:theorem_existence}
         There exist $\lr \in \mathbb{R}^+, \momCoeff \in \mathbb{C}$ so Algorithm~\ref{alg:simultaneous_complex} converges for NAMEx-Momentum.
\end{restatable}

\section{Experimental Results}
\label{sec:experiment}
\begin{wrapfigure}[16]{r}{0.4\textwidth}
\vspace{-0.2in}
    \centering
    \includegraphics[width=0.4\textwidth]{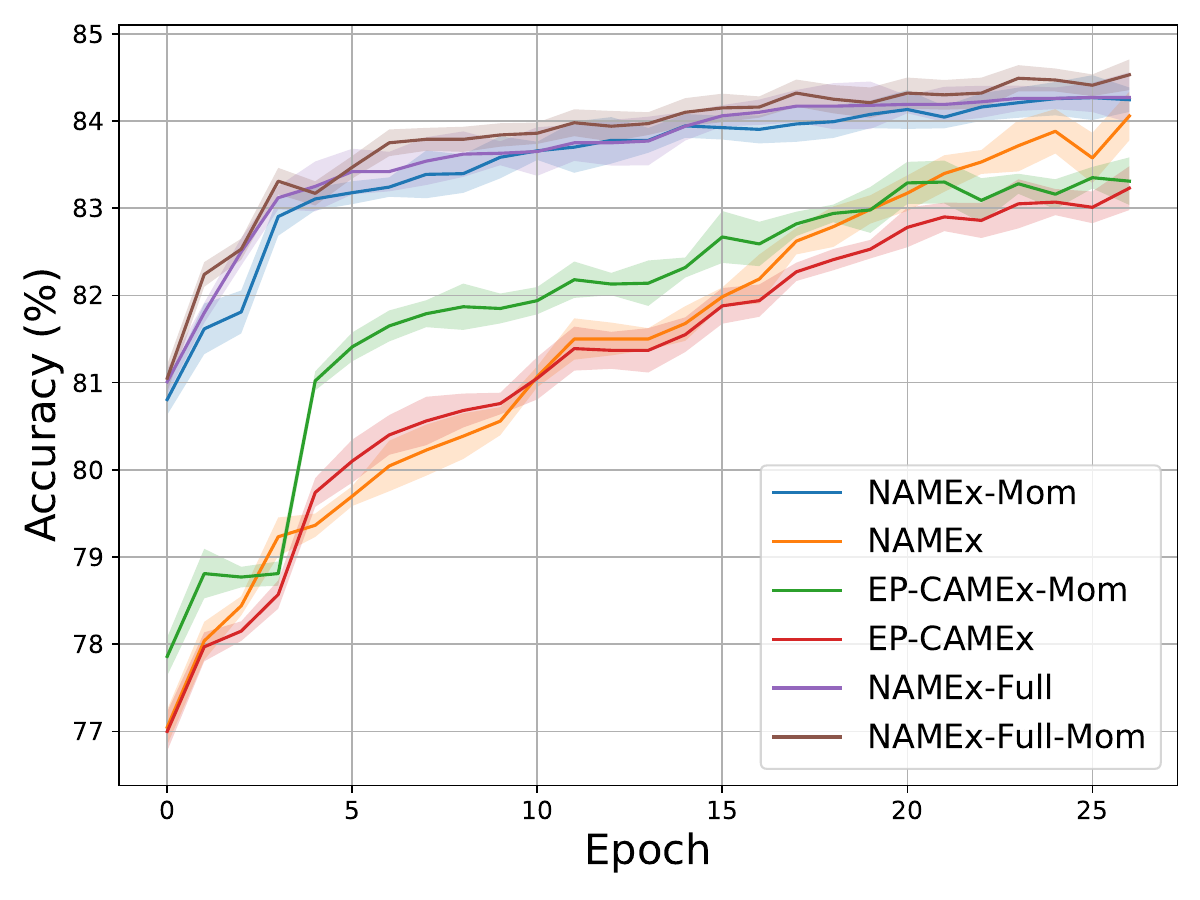}
    \vspace{-2em}
    \caption{Top-1 Accuracy Evaluation of Swin Transformer Variants. Complex momentum enhances convergence speed and performance of both NAMEx and EP-CAMEx.}
    \label{fig:swin-acc}
     \vspace{-0.5in}
\end{wrapfigure}
We evaluate NAMEx and its variants against baseline methods (SMoE, CAMEx, and EP-CAMEx) across diverse tasks: language modeling (WikiText-103~\citep{merity2016pointer}), text classification (GLUE~\citep{wang2019glue}), and image classification (ImageNet-1K~\citep{deng2009imagenet}). To assess robustness, we include evaluations on corrupted datasets: ImageNet-A, ImageNet-O, and ImageNet-R~\citep{hendrycks2021natural, hendrycks2021many}. Results, averaged over five random seeds, show that: (1) NAMEx, leveraging Nash bargaining, consistently improves performance on both vision and language benchmarks; and (2) complex momentum provides additional gains. For MLP-based experts, we follow standard practice and merge parameters layer-wise~\citep{yadav2023ties-merging, yu2024dare, matena2022merging}. Experiments are run on a 8×A100 server. Additional details are available in Appendix~\ref{app:dataset-model-details}.

To match EP-CAMEx's training time, we fix the bargaining budget to 20 iterations per batch and evaluate two NAMEx variants: (1) NAMEx, which computes $\bf \alpha$ once at the first layer and reuses it, showing strong performance over naive averaging; and (2) NAMEx-Full, which distributes the budget evenly across layers, inspired by~\citep{navon2022multi}. Update strategies are further discussed in Section~\ref{sec:ablation}.

In the tables that follow, NAMEx-Full results are highlighted in \textit{\color{blue!30}grey}, with the best and second-best scores shown in \textbf{bold} and \underline{underlined}, respectively.

\subsection{Language Modeling}
\begin{table}[t!]
    \centering
    \caption{Validation and test perplexity on WikiText-103 for small- and medium-scale pretraining.}
    \label{tab:merged-wikitext}
    \resizebox{0.65\linewidth}{!}{
    \begin{NiceTabular}{lcccccc}
        \toprule
        \textbf{Model} & \textbf{Params} & \multicolumn{2}{c}{\textbf{Small}} & \multicolumn{2}{c}{\textbf{Medium}} \\
        \cmidrule(lr){3-4} \cmidrule(lr){5-6}
        & & \textbf{Val PPL} & \textbf{Test PPL} & \textbf{Val PPL} & \textbf{Test PPL} \\
        \midrule
        SMoE (Top-1)          & 70M / 216M & $86.64 \scriptscriptstyle{\pm .22}$ & $87.79 \scriptscriptstyle{\pm .31}$ & $38.60 \scriptscriptstyle{\pm .18}$ & $40.51 \scriptscriptstyle{\pm .25}$ \\
        \RowStyle{}SMoE (Top-2)         & 70M / 216M & $84.26 \scriptscriptstyle{\pm .12}$ & $ 84.81 \scriptscriptstyle{\pm .29}$ & $33.76 \scriptscriptstyle{\pm .19}$ & $35.55 \scriptscriptstyle{\pm .22}$ \\
        SMEAR          & 70M / 216M & $85.56 \scriptscriptstyle{\pm .20}$ & $87.24 \scriptscriptstyle{\pm .28}$ & $36.15 \scriptscriptstyle{\pm .17}$ & $37.42 \scriptscriptstyle{\pm .23}$ \\
        CAMEx          & 70M / 216M & $83.53 \scriptscriptstyle{\pm .19}$ & $84.48 \scriptscriptstyle{\pm .26}$ & $35.69 \scriptscriptstyle{\pm .15}$ & $36.53 \scriptscriptstyle{\pm .21}$ \\
        \midrule
        \multicolumn{6}{c}{\textit{w/o momentum}} \\
        \midrule
        EP-CAMEx       & 70M / 216M & $83.89 \scriptscriptstyle{\pm .18}$ & $85.03 \scriptscriptstyle{\pm .24}$ & $35.78 \scriptscriptstyle{\pm .16}$ & $36.55 \scriptscriptstyle{\pm .22}$ \\
        NAMEx          & 70M / 216M & $83.30 \scriptscriptstyle{\pm .21}$ & $84.12 \scriptscriptstyle{\pm .29}$ & $35.14 \scriptscriptstyle{\pm .19}$ & $36.40 \scriptscriptstyle{\pm .27}$ \\
        \rowcolor{blue!10}
        NAMEx-Full     & 70M / 216M & $82.85 \scriptscriptstyle{\pm .17}$ & $83.16 \scriptscriptstyle{\pm .23}$ & $34.92 \scriptscriptstyle{\pm .14}$ & $36.21 \scriptscriptstyle{\pm .20}$ \\
        \midrule
        \multicolumn{6}{c}{\textit{w/ momentum}} \\
        \midrule
        EP-CAMEx-Mom   & 70M / 216M & $82.90 \scriptscriptstyle{\pm .16}$ & $84.05 \scriptscriptstyle{\pm .22}$ & $35.09 \scriptscriptstyle{\pm .13}$ & $36.16 \scriptscriptstyle{\pm .19}$ \\
        NAMEx-Mom      & 70M / 216M & $\underline{82.63 \scriptscriptstyle{\pm .15}}$    & $\underline{83.59 \scriptscriptstyle{\pm .21}}$    & $\underline{34.89 \scriptscriptstyle{\pm .12}}$    & $\underline{35.86 \scriptscriptstyle{\pm .18}}$ \\
        \rowcolor{blue!10}
        NAMEx-Full-Mom & 70M / 216M & $\mathbf{82.44 \scriptscriptstyle{\pm .14}}$ & $\mathbf{82.94 \scriptscriptstyle{\pm .20}}$ & $\mathbf{34.25 \scriptscriptstyle{\pm .11}}$  & $\mathbf{35.37 \scriptscriptstyle{\pm .17}}$ \\
        \bottomrule
    \end{NiceTabular}
}
\end{table}

\begin{table}[t!]
    \centering
	\caption{Performance of T5-base variants on fine-tuning tasks for GLUE. All SMoE variants have 8 experts per layer. Following \citep{devlin2019bert}, we conduct experiments on the GLUE benchmark.}
	\label{tab:glue-results} 
    \resizebox{0.85\linewidth}{!}{
      \setlength{\tabcolsep}{2pt}
      \begin{NiceTabular}{lcccccccc}
            \toprule
            \textbf{Model} & \textbf{Params} & \textbf{SST-2} & \textbf{MRPC} & \textbf{CoLA} & \textbf{STSB} & \textbf{RTE} & \textbf{QNLI} & \textbf{MNLI} \\
            \midrule
            Dense & 220M & $93.34 \scriptscriptstyle{\pm .15}$ & $89.70 \scriptscriptstyle{\pm .11}$ & $58.06 \scriptscriptstyle{\pm .15}$ & $89.06 \scriptscriptstyle{\pm .22}$ & $74.36 \scriptscriptstyle{\pm .27}$ & $92.34 \scriptscriptstyle{\pm .14}$ & $86.36 \scriptscriptstyle{\pm .15}$ \\
            SMoE (Top-1) & 1.0B & $94.26 \scriptscriptstyle{\pm .13}$ & $90.87 \scriptscriptstyle{\pm .12}$ & $56.78 \scriptscriptstyle{\pm .24}$ & $89.44 \scriptscriptstyle{\pm .29}$ & $70.75 \scriptscriptstyle{\pm .32}$ & $92.07 \scriptscriptstyle{\pm .13}$ & $86.38 \scriptscriptstyle{\pm .17}$ \\
            \RowStyle{} SMoE (Top-2) & 1.0B 
            & $94.35 \scriptscriptstyle{\pm .14}$ 
            & $91.04 \scriptscriptstyle{\pm .12}$ 
            & $58.43 \scriptscriptstyle{\pm .26}$ 
            & $89.73 \scriptscriptstyle{\pm .28}$ 
            & $74.98 \scriptscriptstyle{\pm .29}$ 
            & $92.48 \scriptscriptstyle{\pm .16}$ 
            & $86.72 \scriptscriptstyle{\pm .15}$ \\
            CAMEx & 1.0B & $93.80 \scriptscriptstyle{\pm .14}$ & $91.16 \scriptscriptstyle{\pm .13}$ & $58.57 \scriptscriptstyle{\pm .24}$ & $89.47 \scriptscriptstyle{\pm .23}$ & $74.72 \scriptscriptstyle{\pm .35}$ & $92.60 \scriptscriptstyle{\pm .19}$ & $86.44 \scriptscriptstyle{\pm .12}$ \\ \midrule
            \multicolumn{9}{c}{\textit{w/o momentum}} \\ \midrule
            EP-CAMEx & 1.0B & $93.69 \scriptscriptstyle{\pm .11}$ & $91.01 \scriptscriptstyle{\pm .14}$ & $58.29 \scriptscriptstyle{\pm .24}$ & $89.92 \scriptscriptstyle{\pm .31}$ & $75.81 \scriptscriptstyle{\pm .33}$ & $92.17 \scriptscriptstyle{\pm .15}$ & $86.94 \scriptscriptstyle{\pm .14}$ \\
            NAMEx & 1.0B & $94.46 \scriptscriptstyle{\pm .12}$ & $92.01 \scriptscriptstyle{\pm .14}$ & $58.81 \scriptscriptstyle{\pm .36}$ & $90.12 \scriptscriptstyle{\pm .33}$ & $75.09 \scriptscriptstyle{\pm .22}$ & $92.86 \scriptscriptstyle{\pm .17}$ & $86.96 \scriptscriptstyle{\pm .12}$ \\ 
            \rowcolor{blue!10}
            NAMEx-Full & 1.0B & $\underline{94.82 \scriptscriptstyle{\pm .15}}$ & $92.80 \scriptscriptstyle{\pm .13}$ & $\underline{59.63 \scriptscriptstyle{\pm .22}}$ & $\underline{90.27 \scriptscriptstyle{\pm .24}}$ & $\underline{77.83 \scriptscriptstyle{\pm .31}}$ & $\underline{93.23 \scriptscriptstyle{\pm .18}}$ & $\underline{87.23 \scriptscriptstyle{\pm .14}}$ \\ \midrule
            \multicolumn{9}{c}{\textit{w/ momentum}} \\ \midrule
            EP-CAMEx-Mom & 1.0B & $94.61 \scriptscriptstyle{\pm .17}$ & $92.47 \scriptscriptstyle{\pm .13}$ & $59.31 \scriptscriptstyle{\pm .25}$ & $90.07 \scriptscriptstyle{\pm .23}$ & $76.17 \scriptscriptstyle{\pm .36}$ & $92.99 \scriptscriptstyle{\pm .13}$ & $86.80 \scriptscriptstyle{\pm .15}$ \\
            NAMEx-Mom & 1.0B & $94.61 \scriptscriptstyle{\pm .14}$ & $\underline{93.02 \scriptscriptstyle{\pm .16}}$ & $58.90 \scriptscriptstyle{\pm .41}$ & $90.06 \scriptscriptstyle{\pm .36}$ & $77.62 \scriptscriptstyle{\pm .37}$ & $93.11 \scriptscriptstyle{\pm .10}$ & $87.02 \scriptscriptstyle{\pm .14}$ \\
            \rowcolor{blue!10}
            NAMEx-Full-Mom & 1.0B & $\mathbf{95.06 \scriptscriptstyle{\pm .12}}$ & $\mathbf{93.27 \scriptscriptstyle{\pm .14}}$ & $\mathbf{60.13 \scriptscriptstyle{\pm .32}}$ & $\mathbf{90.63 \scriptscriptstyle{\pm .27}}$ & $\mathbf{78.15 \scriptscriptstyle{\pm .30}}$ & $\mathbf{93.31 \scriptscriptstyle{\pm .14}}$  & $\mathbf{87.45 \scriptscriptstyle{\pm .11}}$ \\
            \bottomrule
        \end{NiceTabular}
        }
\end{table}

We adopt the experimental setup of \citep{pham2024competesmoeeffectivetraining} and \citep{teo2024momentumsmoe} for pre-training and evaluating on the WikiText-103 dataset. \cref{tab:merged-wikitext} presents the results of our methods on small-scale and medium-scale pre-training tasks using WikiText-103. For both small- and medium-scale pre-training, NAMEx-Full-Mom achieves the lowest validation/test perplexities, outperforming SMoE and CAMEx-based methods. NAMEx variants as well as momentum-equipped variants consistently surpass their counterparts across scales, proving the efficacy of Nash bargaining and momentum integration.

\subsection{Text Classification}
We evaluate our method on downstream text classification tasks using the GLUE dataset \citep{wang2019glue}, with all models built on the T5-Base backbone. As shown in \cref{tab:glue-results}, NAMEx-Full-Mom achieves the best results on all tasks. NAMEx consistently outperforms SMoE {(Top-1 and Top-2 routing)}, CAMEx, and EP-CAMEx, highlighting the effectiveness of the Nash bargaining solution. The momentum-based extensions, NAMEx-Mom and EP-CAMEx-Mom, further enhance performance, demonstrating improved robustness and generalization.

\subsection{Image Classification}
\begin{table}[t!]
    \centering
    \caption{Finetuning and zero-shot results on ImageNet-1k and corrupted variants.}
    \label{tab:imagenet-results}
    \resizebox{0.8\linewidth}{!}{
    \begin{tabular}{lcccccc}
        \toprule
        \textbf{Model} & \textbf{Params} & \textbf{Acc@1} & \textbf{Acc@5} & \textbf{INet-O} & \textbf{INet-A} & \textbf{INet-R} \\
        \midrule
        SMoE & 50M & $83.15 \scriptscriptstyle{\pm .17}$ & $96.71 \scriptscriptstyle{\pm .12}$ & $43.34 \scriptscriptstyle{\pm .21}$ & $23.72 \scriptscriptstyle{\pm .18}$ & $38.02 \scriptscriptstyle{\pm .20}$ \\
        SMEAR & 50M & $83.15 \scriptscriptstyle{\pm .14}$ & $96.91 \scriptscriptstyle{\pm .09}$ & $43.35 \scriptscriptstyle{\pm .19}$ & $24.14 \scriptscriptstyle{\pm .16}$ & $38.16 \scriptscriptstyle{\pm .22}$ \\

        CAMEx & 50M & $83.29 \scriptscriptstyle{\pm .24}$ & $96.95 \scriptscriptstyle{\pm .13}$ & $50.69 \scriptscriptstyle{\pm .25}$ & $25.45 \scriptscriptstyle{\pm .21}$ & $38.37 \scriptscriptstyle{\pm .20}$ \\
        \midrule
        \multicolumn{7}{c}{\textit{w/o momentum}} \\
        \midrule
        EP-CAMEx & 50M & $83.23 \scriptscriptstyle{\pm .25}$ & $96.93 \scriptscriptstyle{\pm .16}$ & $50.27 \scriptscriptstyle{\pm .28}$ & $24.22 \scriptscriptstyle{\pm .17}$ & $37.88 \scriptscriptstyle{\pm .23}$ \\
        NAMEx & 50M & $84.06 \scriptscriptstyle{\pm .28}$ & $97.19 \scriptscriptstyle{\pm .18}$ & $50.30 \scriptscriptstyle{\pm .27}$ & $25.32 \scriptscriptstyle{\pm .15}$ & $38.56 \scriptscriptstyle{\pm .19}$ \\
        \rowcolor{blue!10}
        NAMEx-Full & 50M & $84.27 \scriptscriptstyle{\pm .24}$ & \underline{$97.94 \scriptscriptstyle{\pm .14}$} & $50.66 \scriptscriptstyle{\pm .22}$ & $25.74 \scriptscriptstyle{\pm .16}$ & $38.70 \scriptscriptstyle{\pm .18}$ \\
        \midrule
        \multicolumn{7}{c}{\textit{w/ momentum}} \\
        \midrule
        EP-CAMEx-Mom & 50M & $83.56 \scriptscriptstyle{\pm .12}$ & $97.03 \scriptscriptstyle{\pm .11}$ & $50.37 \scriptscriptstyle{\pm .20}$ & $33.22 \scriptscriptstyle{\pm .24}$ & $38.22 \scriptscriptstyle{\pm .19}$ \\
        NAMEx-Mom & 50M & \underline{$84.28 \scriptscriptstyle{\pm .26}$} & \underline{$97.94 \scriptscriptstyle{\pm .12}$} & \underline{$51.22 \scriptscriptstyle{\pm .18}$} & \underline{$35.05 \scriptscriptstyle{\pm .19}$} & \underline{$38.82 \scriptscriptstyle{\pm .14}$} \\
        \rowcolor{blue!10}
        NAMEx-Full-Mom & 50M & $\mathbf{84.52 \scriptscriptstyle{\pm .18}}$ & $\mathbf{98.11 \scriptscriptstyle{\pm .15}}$ & $\mathbf{51.34 \scriptscriptstyle{\pm .17}}$ & $\mathbf{35.27 \scriptscriptstyle{\pm .20}}$ & $\mathbf{38.96 \scriptscriptstyle{\pm .13}}$ \\
        \bottomrule
    \end{tabular}
}
\vspace{-0.1in}
\end{table}

In this section, we evaluate our method on image classification tasks using the Swin-Transformer \citep{liu2021Swin} and its MoE variant \citep{hwang2023tuteladaptivemixtureofexpertsscale}. Specifically, we fine-tune Swin-MoE Small on ImageNet-1k, training all models for 30 epochs with a batch size of 96. For each MoE layer, we perform Algorithm \ref{alg:simultaneous_complex}, where apart from the first MoE layer that an $\B{E}_m$ expert is initialized, all experts are merged into $\B{E}_m$. We further evaluate NAMEx on another SMoE architecture ACMoE \citep{nielsen2025tight}, we follow ACMoE training configurations, i.e., we train the NAMEx variants on top of the ACMoE backbone for 100 epochs with batchsize 512.

\cref{tab:imagenet-results} shows NAMEx-Mom outperforming all baselines, with NAMEx close behind; even without momentum, NAMEx-Full matches NAMEx-Mom on clean benchmarks, confirming the value of layer-wise Nash solutions. Across distribution shifts (ImageNet-A/O/R~\citep{hendrycks2021many,hendrycks2021natural}), NAMEx-Mom achieves the best zero-shot accuracy, with momentum variants showing the strongest gains, especially on ImageNet-A.

{ In \cref{tab:imagenet-acmoe-results} of Appendix~\cref{sec:other_results}, across all ImageNet variants, the NAMEx-based models consistently outperforms the ACMoE Top-1 and Top-2 baselines.  In particular, NAMEx-Full and NAMEx-Full-Mom set new best accuracies on both in-distribution metrics (Acc@1 and Acc@5)  and out-of-distribution benchmarks (INet-O, INet-A, INet-R). This underlines the strong generalization ability of NAMEx.  Even with the same parameter budget, NAMEx variants deliver better robustness to corruptions and distribution shifts.}


\subsection{Zero-shot and Finetuning on DeepSeek-MoE (16B) and Qwen1.5-MoE (14B)}
We test NAMEx-Full at scale by integrating it into DeepSeek-MoE \citep{liu2024deepseek} (16B parameters, 1 shared expert, and 63 routed experts) and Qwen1.5-MoE (14B parameters), evaluating in both zero-shot and SmolTalk fine-tuned settings. As shown in Table~\ref{tab:routing_merge} below (for DeepSeek-MoE) and Table~\ref{tab:qwen_zeroshot},~\ref{tab:qwen_finetune} in Appendix~\ref{app:qwen} (for Qwen1.5-MoE), NAMEx-Full consistently outperforms the baselines and EP-CAMEx across routing strategies and benchmarks (MMLU \citep{hendrycks2021measuring}, GSM8K \citep{cobbe2021trainingverifierssolvemath} and ARC \citep{clark2018thinksolvedquestionanswering}), demonstrating robust and generalizable gains in expert collaboration.
\begin{table}[t!]
\centering
\caption{Performance comparison across routing strategies and models on MMLU, GSM8K, and ARC benchmarks. Left: original results. Right: fine-tuned \deept{} variants on SmolTalk.}
\setlength{\tabcolsep}{6pt}
\resizebox{0.9\linewidth}{!}{
\begin{NiceTabular}{llcccccc}
\toprule
\multirow{2}{*}{\textbf{Routing Strategy}} & \multirow{2}{*}{\textbf{Model}} 
& \multicolumn{3}{c}{\textbf{Zero-Shot}} 
& \multicolumn{3}{c}{\textbf{Fine-tuned (SmolTalk)}} \\ 
\cmidrule(lr){3-5} \cmidrule(lr){6-8}
& & \textbf{MMLU} & \textbf{GSM8K} & \textbf{ARC} 
  & \textbf{MMLU} & \textbf{GSM8K} & \textbf{ARC} \\
\midrule
\multirow{3}{*}{Linear} 
  & Deepseek-MoE & 44.77 & 16.53 & 49.15 & 45.21 & 17.10 & 49.50 \\
  & EP-CAMEx     & 44.85 & 16.63 & 49.26 & 45.33 & 17.24 & 49.62 \\
  \rowcolor{blue!10}& \textbf{NAMEx-Full} ($0$ disagreement point) & \textbf{44.92} & \textbf{16.77} & \textbf{49.51} 
     & \textbf{45.47} & \textbf{17.36} & \textbf{49.85} \\
\RowStyle{}& \textbf{NAMEx-Full} (mean disagreement point) & \textbf{44.93} & \textbf{16.75} & \textbf{49.52} 
                         & \textbf{45.47} & \textbf{17.39} & \textbf{49.84} \\
\midrule
\multirow{3}{*}{Cosine} 
  & Deepseek-MoE & 44.95 & 16.70 & 49.30 & 45.34 & 17.25 & 49.60 \\
  & EP-CAMEx     & 45.05 & 16.81 & 49.40 & 45.45 & 17.32 & 49.73 \\
\rowcolor{blue!10}
  & \textbf{NAMEx-Full} ($0$ disagreement point) & \textbf{45.10} & \textbf{16.88} & \textbf{49.60} 
                       & \textbf{45.66} & \textbf{17.53} & \textbf{49.92} \\
 \RowStyle{} & \textbf{NAMEx-Full} (mean disagreement point) & \textbf{45.09} & \textbf{16.89} & \textbf{49.58} 
                         & \textbf{45.67} & \textbf{17.52} & \textbf{49.92} \\
\midrule
\multirow{3}{*}{Stable-MoE} 
  & Deepseek-MoE & 45.80 & 17.50 & 49.90 & 46.17 & 17.63 & 50.28 \\
  & EP-CAMEx     & 45.88 & 17.62 & 50.00 & 46.25 & 18.10 & 50.45 \\
\rowcolor{blue!10}
  & \textbf{NAMEx-Full} ($0$ disagreement point) & \textbf{45.95} & \textbf{17.70} & \textbf{50.15} 
                       & \textbf{46.42} & \textbf{18.23} & \textbf{50.64} \\
\RowStyle{} & \textbf{NAMEx-Full} (mean disagreement point) & \textbf{45.92} & \textbf{17.68} & \textbf{50.19} 
                         & \textbf{46.40} & \textbf{18.23} & \textbf{50.63} \\
\bottomrule
\end{NiceTabular}
}
\vspace{-0.11in}
\label{tab:routing_merge}
\end{table}

\section{Empirical Analysis}
\label{sec:ablation}
\begin{wraptable}[12]{r}{0.4\textwidth}
\vspace{-0.2in}
\centering
\caption{Impact of varying the argument of $\beta$ on the performance of EP-CAMEx and NAMEx.}
\label{tab:merged-arg}
\resizebox{1\linewidth}{!}{
\begin{tabular}{cccccc}
\toprule
\makecell{$\phi$} & Model & SST-2 & MRPC & STS-B & RTE \\
\midrule
\multirow{2}{*}{$\pi/6$} & EP-CAMEx & 94.27 & 92.50 & 89.77 & 76.17 \\
 & NAMEx & 94.83 & 92.82 & 89.68 & 76.42 \\ \midrule
\multirow{2}{*}{$\pi/12$} & EP-CAMEx & 94.61 & 92.42 & 89.53 & 76.23 \\
 & NAMEx & 93.92 & 92.69 & 89.55 & 76.53 \\\midrule
\multirow{2}{*}{$0$} & EP-CAMEx & 93.45 & 92.24 & 89.51 & 72.20 \\
 & NAMEx & 93.56 & 91.66 & 89.51 & 75.09 \\\midrule
\multirow{2}{*}{$-\pi/12$} & EP-CAMEx & 93.56 & 91.91 & 89.53 & 76.03 \\
 & NAMEx & 93.92 & 92.60 & 90.15 & 75.57 \\\midrule
\multirow{2}{*}{$-\pi/6$} & EP-CAMEx & 94.72 & 91.93 & 89.38 & 72.56 \\
 & NAMEx & 94.50 & 92.75 & 89.45 & 76.64 \\
\bottomrule
\end{tabular}
}
\end{wraptable}
\textbf{Synthetic Example.} To illustrate how NAMEx encourages balanced expert cooperation, we construct a toy SMoE model with three experts per layer. Utility trade-offs are visualized in a 3D space, where each axis represents the utility of one expert. Figure~\ref{fig:pareto-plot}, Appendix~\ref{app:extra_emp_analysis}, shows that average-based expert merging may fail to reach the Pareto set, whereas NAMEx tends to produce more Pareto-efficient outcomes. This illustrates that NAMEx is not dominated by EP-CAMEx or linear average merging.

\textbf{Number of Optimization Steps.} 
\Cref{tab:num-optim} shows that smaller step frequencies ($\Delta l=1,2,5$) often improve or match baseline performance ($\Delta l=L$) but at the cost of higher runtime (0.69s → 4.70s), underscoring a trade-off between accuracy and efficiency.

\textbf{Impact of Momentum $\mu$ and Step size $\gamma$.} The results in \cref{tab:merged-arg} demonstrate the impact of varying the argument $\phi$ of $\beta$ (fixed modulus $0.9$). All tasks show declines at $\phi=0$ (real momentum), suggesting non-zero arguments are critical. NAMEx consistently outperforms EP-CAMEx, reaffirming its robustness. Optimal results require task-specific $\phi$ tuning. 

Finally, RTE exhibits higher sensitivity, peaking at $\phi=\pi/12$ and dropping sharply at $\phi=0$ and $\phi=-\pi/6$, highlighting task-specific $\phi$ dependencies. For analysis on step size $\gamma$, please refer to Figure \ref{fig:gamma} in Appendix~\ref{app:extra_emp_analysis}.

\begin{table}[t!]
\centering
\begin{minipage}[t]{0.52\textwidth}
    \centering
    \caption{Impact of varying the frequency of merging weight update steps on the performance of NAMEx.}
    \label{tab:num-optim}
    \resizebox{\linewidth}{!}{
    \begin{tabular}{lccccc}
    \toprule
    $\Delta l$ & SST-2 & MRPC & STS-B & RTE & Runtime (sec) \\ \midrule
    1   & 94.88 & 92.85 & 90.32 & 77.26 & 4.70\\ 
    2   & 94.95 & 92.38 & 90.37 & 76.89 & 2.29\\ 
    5   & 95.18 & 92.09 & 90.13 & 77.98 & 1.14\\ 
    $L$ & 94.46 & 92.01 & 90.12 & 75.09 & 0.69 \\ 
    \bottomrule
    \end{tabular}
    }
\end{minipage}%
\hfill
\begin{minipage}[t]{0.45\textwidth}
    \centering
    \caption{Performance comparison between complex momentum and quaternion momentum.}
    \label{tab:quarternion}
    \resizebox{0.89\linewidth}{!}{
    \begin{tabular}{lccc}
        \toprule
        Model & MRPC & STS-B & RTE \\
        \midrule
        EP-CAMEx-Mom & 92.47 & 90.07 & 76.17 \\
        NAMEx-Mom    & 93.02 & 90.06 & 77.62 \\ \midrule
        EP-CAMEx-Q   & 92.52 & 90.35 & 77.12 \\
        \rowcolor{blue!10}
        NAMEx-Q      & \textbf{93.24} & \textbf{90.72} & \textbf{77.86} \\
        \bottomrule
    \end{tabular}
    }
\end{minipage}
\vspace{-0.1in}
\end{table}

\textbf{Beyond Complex Momentum.} Quaternions generalize complex numbers with richer 4D dynamics, enabling a quaternion momentum update $z_{t+1}=\beta z_t+\nabla f(x_t)$. While more complex and harder to tune, quaternion momentum can better stabilize high-dimensional optimization and handle rotations. As shown in Table~\ref{tab:quarternion}, choosing $\beta=0.8+0.3\mathbf{i}+0.3\mathbf{j}+0.3\mathbf{k}$ outperforms complex momentum, suggesting multi-buffer momentum is a promising direction with careful hyperparameter tuning.

\begin{table}[ht!]
\centering
\caption{Performance comparison across number of NBS solving iterations for the Linear router in NAMEx-Full config (Qwen1.5-MoE and Deepseek-MoE). All results are slightly improved while maintaining marginal gaps. Throughout experiments, we use 2 CCP iterations per layer for NAMEx-Full (chosen config below).}
\setlength{\tabcolsep}{6pt}

\resizebox{0.8\linewidth}{!}{
\begin{tabular}{llcccccc}
\toprule
\multirow{2}{*}{\textbf{Model}} & \multirow{2}{*}{\textbf{No. Iterations}} 
& \multicolumn{3}{c}{\textbf{Zero-Shot}} 
& \multicolumn{3}{c}{\textbf{Fine-tuned (SmolTalk)}} \\ 
\cmidrule(lr){3-5} \cmidrule(lr){6-8}
& & \textbf{MMLU} & \textbf{GSM8K} & \textbf{ARC} 
  & \textbf{MMLU} & \textbf{GSM8K} & \textbf{ARC} \\
\midrule
\multirow{5}{*}{Qwen1.5-MoE} 
& 2 (Chosen Config)  & 61.87 & 60.55 & 50.95 & 62.10 & 61.00 & 51.35 \\
& 5                   & 61.70 & 60.55 & 50.94 & 62.20 & 61.05 & 51.31 \\
& 20                  & 61.94 & 60.57 & 50.96 & 62.14 & 61.03 & 51.34 \\
& 40                  & 61.92 & 60.62 & 51.01 & 62.22 & 61.05 & 51.46 \\
& 60                  & 61.81 & 60.48 & 51.08 & 62.15 & 60.98 & 51.39 \\
\midrule
\multirow{5}{*}{Deepseek-MoE} 
& 2 (Chosen Config)  & 45.05 & 16.86 & 49.58 & 45.63 & 17.47 & 49.92 \\
& 5                  & 44.84 & 16.86 & 49.57 & 45.73 & 17.51 & 49.88 \\
& 20                 & 45.15 & 16.87 & 49.60 & 45.63 & 17.47 & 49.92 \\
& 40                 & 45.16 & 16.93 & 49.66 & 45.73 & 17.55 & 50.03 \\
& 60                 & 44.93 & 16.77 & 49.72 & 45.68 & 17.46 & 49.96 \\
\bottomrule
\end{tabular}
}
\label{tab:routing_linear_only}
\end{table}

{
\paragraph{Number of CCP iterations.}
Tab.~\ref{tab:routing_linear_only} reports the bargaining budget ablation. In the zero-shot setting, performance remains within a narrow range: MMLU fluctuates between 44.8--45.2 (peaking at 40 iterations), GSM8K shows a slight increase up to 40 before dropping at 60, and ARC improves marginally with a best at 60. After SmolTalk fine-tuning, the curves flatten further--MMLU is tied at 5 and 40 iterations, while GSM8K and ARC peak at 40 with minimal variation overall. Overall, 20--40 iterations slightly outperform 2--5 in some cases, whereas 60 provides no consistent benefit and can degrade performance. Given the small margins and likely variance, we recommend 2 iterations for efficiency and 20 when marginal gains are desired.

\paragraph{Roburtness to choices of disagreement point.} In Tab. \ref{tab:routing_merge} and Tab. \ref{tab:qwen_routing_merge} in \cref{sec:other_results}, we tried "mean" (standard average merging) as the disagreement point and compared it to 0. Across Linear, Cosine, and Stable-MoE, the deltas are tiny (about 0.04 on any metric), with no consistent winner. This shows the gains come from the bargaining weights, not the fallback choice. We keep 0 as the default because it is conservative, stable, and easy to interpret, and it leaves compute unchanged.  
}

\section{Limitation and Conclusion}
\label{sec:conclusion}
In this work, we address expert merging through game theory by proposing NAMEx, a method that integrates Nash Bargaining for equitable collaboration and complex momentum to accelerate convergence with theoretical stability guarantees. Experiments across diverse tasks demonstrate NAMEx's consistent superiority over existing methods, highlighting its adaptability to diverse tasks. While NAMEx could be extended to token-level momentum-based methods, such as \citep{teo2024momentumsmoe} and \citep{puigcerver2024softmoe}, the computational cost of solving the Nash equilibrium per token remains a challenge, leaving this as an avenue for future work.

\subsubsection*{Acknowledgments}
This research / project is supported by the National Research Foundation Singapore under the AI
Singapore Programme (AISG Award No: AISG2-TC-2023-012-SGIL). This research / project is
supported by the Ministry of Education, Singapore, under the Academic Research Fund Tier 1
(FY2023) (A-8002040-00-00, A-8002039-00-00). This research / project is also supported by the
NUS Presidential Young Professorship Award (A-0009807-01-00) and the NUS Artificial Intelligence
Institute–Seed Funding (A-8003062-00-00).

We would like to thank FPT for their generous support in providing discounted GPU access on FPT AI Factory, which enabled the experiments in this work.


\textbf{Ethics Statement.} Given the nature of the work, we do not foresee any negative societal and ethical
impacts of our work.

\textbf{Reproducibility Statement.} Source codes for our experiments are provided in the supplementary
materials of the paper. The details of our experimental settings and computational infrastructure are
given in Section \ref{sec:experiment}, Section~\ref{sec:ablation}, and the Appendix. All datasets that we used in the paper are published, and they
are easy to access in the Internet.

\textbf{LLM Usage Declaration.} We use large language models (LLMs) for grammar checking and correction.
\bibliography{iclr2026_conference}
\bibliographystyle{iclr2026_conference}

\newpage
\appendix
\onecolumn
\begin{center}
{\bf \Large{Supplement to ``Expert Merging in Sparse Mixture of Experts with Nash Bargaining''}}
\end{center}
\label{sec:appendix}
\DoToC

\newpage
\section{Notation} \label{sec:notation}
            \begin{table}[ht!]
                \caption{Notation}
                \vskip 0.1in
                \begin{center}
                    \begin{tabular}{c c}
                        \toprule
                        $\alpha = [\alpha_1, \alpha_2, \dots, \alpha_N ]$ & The Nash coefficients for merging experts. \\
                        $\boldsymbol{x}, \boldsymbol{y}, \boldsymbol{z}, \dots \in \mathbb{C}^{n}$ & Vectors\\
                        $\boldsymbol{X}, \boldsymbol{Y}, \boldsymbol{Z}, \dots \in \mathbb{C}^{n \times n}$ & Matrices\\
                        $\boldsymbol{X}^\top$ & The transpose of matrix $\boldsymbol{X}$\\
                        $\identity$ & The identity matrix\\
                        $\Re(z), \Im(z)$ & The real or imaginary component of $z \in \mathbb{C}$\\
                        $i$ & The imaginary unit. $z \in \mathbb{C} \implies z = \Re(z) + i \Im(z)$\\
                        $\bar{z}$ & The complex conjugate of $z \in \mathbb{C}$\\
                        $|z| := \sqrt{z \bar{z}}$ & The magnitude or modulus of $z \in \mathbb{C}$\\
                        $\arg(z)$ & The argument or phase of $z \in \mathbb{C} \implies z = |z| \exp(i\arg(z))$\\
                        $\bothParam^{(l)}_m \in \mathbb{R}^{\bothDim}$ & Parameters of the base experts at the $l$-th layer of the network\\
                        $\bothParam^l_i \in \mathbb{R}^{\bothDim}$ & Parameters of the $i$-th experts at the $l$-th layer of the network\\
                        $\tau_i^{(l)} = \bothParam^{(l)}_i - \bothParam^{(l)}_m$ & The domain-vector of the $i$-th experts at the $l$-th layer of the network \\
                        $\bothGrad \in \mathbb{R}^{\bothDim}$ & Aggregation of the domain-vector for updating the base expert.\\
                        $\bothParam^{(0)}_m \in \mathbb{R}^{\bothDim}$ & The initial base expert parameter at the first layer\\
                        $\lr \in \mathbb{R}^+$ & The step size for the base expert propagation\\
                        $\eta \in \mathbb{R}^+$ & The step size for creating the $\hat\bothParam^{(l)}_m$ expert that is responsible for processing input\\
                        $\momCoeff \in \mathbb{C}$ & The momentum coefficient\\
                        $\momVal \in \mathbb{C}^{\bothDim}$ & The momentum buffer\\
                        $\eigval \in \mathbb{C}$ & Notation for an arbitrary eigenvalue\\
                        \bottomrule
                    \end{tabular}
                \end{center}
                \label{tab:TableOfNotation}
            \end{table}
\section{Related Work}
\label{sec:related}
\textbf{Sparse Mixture of Experts.} SMoE scales efficiently by activating only a subset of parameters per token, favoring horizontal over deep expansion~\citep{shazeer2017outrageously, lepikhin2021gshard, fedus2022switch, tran2025on,teo2025molex}, and improves Transformer efficiency without loss. In parallel, model merging has gained traction for combining open-source models~\citep{ilharco2022patching,matena2022merging,yadav2023ties-merging, rame2023model, cai2023robust, lu2024twinmerging}, with curvature-aware methods like Fisher Information~\citep{matena2022merging, jin2022dataless} improving quality but at high cost. But, most merging approaches assume shared initialization~\citep{yadav2023ties-merging, ilharco2022patching}, conflicting with SMoE’s independently initialized experts and making merging more difficult.

\textbf{Nash Bargaining Game.} Originally introduced by \citep{nash1950barganing, nash1953two}, the Nash bargaining framework has been widely studied~\citep{kalai1975other} and recently applied to multi-task learning~\citep{navon2022multi, pmlr-v202-shamsian23a}. It has also shown success in diverse domains such as multi-armed bandits~\citep{baek2021fair}, clustering~\citep{rezaee2021GBK-means}, distributed computing~\citep{penmatsa2011game}, and economics~\citep{aumann1992handbook, muthoo1999bargaining}.

\textbf{Momentum in Deep Learning.} Momentum-based optimization has been widely studied, from its origins in classical methods~\citep{polyak1964some, nesterov1983method} to its adaptation for deep learning~\citep{sutskever2013importance, zhang2017yellowfin, nguyen2022momentumtransformer, teo2024momentumsmoe}. Beyond optimization, momentum has also been examined from dynamical and architectural perspectives. Continuous-time interpretations such as Heavy Ball Neural ODE \citep{xia2021heavyballnode} and Nesterov Neural ODE~\citep{nguyen2022improving} connect accelerated methods to second-order differential equations, while \citet{nguyen2020momentumrnn} incorporates momentum dynamics directly into recurrent architectures. Scheduled Restart Momentum~\citep{wang2020scheduled} studies restart schemes for accelerated stochastic gradient descent, and recent empirical analyses~\citep{sutskever2013importance, wang2021momentumarchitecture} investigate how momentum shapes representation learning and architectural behavior in deep networks. In game-theoretic settings, \citet{gidel2019negative} explored negative momentum for games, and \citet{lorraine2022complex} introduced complex momentum, extending momentum methods to differentiable games using complex-valued updates.
\section{Proofs of the Main Results}
\label{app:proofs}
\begin{assumption}
We assume a SMoE architecture of infinite SMoE layers with $\{\B{E}^{(l)}_m, \B{E}^{(l)}_1, \B{E}^{(l)}_2,\dots, \B{E}^{(l)}_N \}$ being the epxerts parameters at $l$-th layer. 
\end{assumption}
\begin{assumption}
    The norm of experts parameters is bounded, that is:
    \begin{equation}
        \Vert \B{E}^{(l)}_i \Vert \leq B \quad \forall l \in \{1, 2, \dots, \infty \} \quad \forall i \in \{m, 1, \dots, N \}
    \end{equation}
\end{assumption}

\subsection{Lemma \ref{lem:alphas} Proof Sketch}\label{pf:lm3.1}
\begin{proof}
The derivative of the objective function is 
\[
\sum_{i=1}^N \frac{1}{\bothGrad^\top \tau_i} \tau_i.
\]
For all $\bothGrad$ such that $\bothGrad^\top \tau_i > 0$ for all $i$, the utilities increase monotonically with the norm of $\bothGrad$. Hence, by Nash's Pareto optimality axiom, the optimal solution must lie on the boundary of $B_\epsilon$. At the optimal point, the gradient 
\[
\sum_{i=1}^N \frac{1}{\bothGrad^\top \tau_i} \tau_i
\]
must be in the radial direction, i.e.,
\[
\sum_{i=1}^N \frac{1}{\bothGrad^\top \tau_i} \tau_i \propto \bothGrad.
\]
Equivalently, there exists $\lambda > 0$ such that
\[
\sum_{i=1}^N \frac{1}{\bothGrad^\top \tau_i} \tau_i = \lambda \bothGrad.
\]
Since the gradients $\tau_i$ are linearly independent, we can express $\bothGrad$ as $\bothGrad = \sum_{i=1}^N \alpha_i \tau_i$. Substituting this into the alignment condition, we obtain
\[
\frac{1}{\bothGrad^\top \tau_i} = \lambda \alpha_i \quad \forall i.
\]
This implies $\bothGrad^\top \tau_i = \frac{1}{\lambda \alpha_i}$. As $\bothGrad^\top \tau_i > 0$ for a descent direction, we deduce $\lambda > 0$. Setting $\lambda = 1$ gives the direction of $\bothGrad$. Thus, finding the Nash bargaining solution reduces to finding $\alpha \in \mathbb{R}^N_+$ such that
\[
\bothGrad^\top \tau_i = \sum_{j=1}^N \alpha_j \tau_j^\top \tau_i = \frac{1}{\alpha_i} \quad \forall i.
\]
This is equivalent to solving $G^\top G \alpha = 1 / \alpha$, where $1 / \alpha$ is the element-wise reciprocal.
\end{proof}

\subsection{Convergence guarantee for NAMEx-Momentum }\label{pf:thm1}
We first have that: 
\begin{equation}
    \bothGrad^{(l)} = \displaystyle\sum_{i=1}^{N}  \alpha_i * \tau_i^{(l)} = \displaystyle\sum_{i=1}^{N} \alpha_i \B{E}^{(l)}_i - \left( \displaystyle\sum_{i=1}^{N} \alpha_i\right) \B{E}_m^{(l)}.
\end{equation}
Given the analogy between expert merging and gradient descent, we apply the formulation of momentum into Eqn.~\ref{eq:dynamic}:
\begin{equation}
 \begin{cases}
 \label{eq:momentum}
    \mathbf{E}_m^{(l+1)} &= \mathbf{E}_m^{(l)} + \lr\bothGrad^{(l)}  + \beta (\B{E}_m^{(l)} - \B{E}_m^{(l-1)}), \\
    \mathbf{\hat{E}}_m^{(l+1)} &= \mathbf{E}_m^{(l+1)} + \eta \displaystyle\sum_{i=1}^{N} \mathbf{M}_i \cdot (s^{(l+1)}_i * \tau_i^{(l+1)}).
\end{cases}   
\end{equation}

Expanding the parameter updates with the Cartesian components of $\lr$ and $\momCoeff$ is key for Theorem~\ref{theo:converge_complex}, which characterizes the convergence rate:
    \begin{align}\label{eq:cartesian_complex_update}
    \momVal^{(l+1)}      &= \momCoeff \momVal^{(l)} + \bothGrad^{(l)} \iff \nonumber\\ 
    \smash{\Re(\momVal^{(l+1)})}   &= \Re(\momCoeff) \Re(\momVal^{(l)}) - \Im(\momCoeff)\Im(\momVal^{(l)})  + \Re(\bothGrad^{(l)})\nonumber \\ 
                         &= \Re(\momCoeff) \Re(\momVal^{(l)}) - \Im(\momCoeff)  \Im(\momVal^{(l)}) +\displaystyle\sum_{i=1}^{N} \alpha_i \B{E}^{(l)}_i - \left( \displaystyle\sum_{i=1}^{N} \alpha_i\right) \B{E}_m^{(l)}, \\
    \Im(\momVal^{(l+1)})             &= \Im(\momCoeff)\Re(\momVal^{{(l)}}) + \Re(\momCoeff)  \Im(\momVal^{(l)}) 
    \end{align}
        \vspace{-0.02\textheight}
        \begin{align}\label{eq:parameter_dynamics}
                \bothParam^{(l+1)}_m &= \bothParam^{(l)}_m + \Re(\lr \momVal^{(l+1)}) \\
                \bothParam^{(l+1)}_m &= \bothParam^{(l)}_m + \lr\bothGrad^{(l)} +\Re(\lr \momCoeff) \Re(\momVal^{(l)}) - \Im(\lr \momCoeff)\Im(\momVal^{(l)})\nonumber \\
                &= \bothParam^{(l)}_m + \lr\displaystyle\sum_{i=1}^{N} \alpha_i \B{E}^{(l)}_i - \lr\left( \displaystyle\sum_{i=1}^{N} \alpha_i\right) \B{E}_m^{(l)} +\Re(\lr \momCoeff) \Re(\momVal^{(l)}) - \Im(\lr \momCoeff)\Im(\momVal^{(l)})
        \end{align}

Setting $\suma = \displaystyle\sum_{i=1}^{N} \alpha_i$, we have
\begin{align} 
    \dynamics = \begin{bmatrix}
        \Re(\momCoeff) \identity & -\Im(\momCoeff) \identity & -\suma \identity\\
        \Im(\momCoeff) \identity & \Re(\momCoeff) \identity & 0\\
        \Re(\lr \momCoeff) \identity & -\Im(\lr \momCoeff) \identity & \identity - \lr\suma\identity\\
        \end{bmatrix}
        \quad 
        \text{and}
        \quad
        \B{q}^{(l)} = \begin{bmatrix}
        \displaystyle\sum_{i=1}^{N} \alpha_i \B{E}^{(l)}_i &
        0 & 
        \lr\displaystyle\sum_{i=1}^{N} \alpha_i \B{E}^{(l)}_i
    \end{bmatrix}^\top
\end{align}

Our parameters evolve with expert-propagation merging via:
\begin{equation}\label{eq:dyn_system}
    [\Re(\momVal^{(l+1)}), \Im(\momVal^{(l+1)}), \bothParam^{(l+1)}]^{\top} =
    \dynamics \, [\Re(\momVal^{(l)}), \Im(\momVal^{(l)}), \bothParam^{(l)}]^{\top} + \B{q}^{(l)}{^\top}
\end{equation}
We can bound convergence rates by looking at the spectral radius of $\dynamics$ with Theorem~\ref{theo:converge_complex}.
\begin{restatable}[Consequence of Prop. 4.4.1 \citep{bertsekas2008nonlinear}]{thm}{complexThm}
\label{theo:converge_complex}
    If the spectral radius $\spectralRadius(\dynamics)  \!<\! 1$, then, for $[\momVal, \bothParam_m]$ in a neighborhood of $[\momVal^*\!, \bothParam_m^*]$, the distance of $[\momVal^{(l)}, \bothParam^{(l)}]$ to the stationary point $[\momVal^*, \bothParam_m^*]$ converges at a linear rate $\mathcal{O}((\spectralRadius(\dynamics) + \epsilon)^l), \forall \epsilon \!>\! 0$.
\end{restatable}
    \begin{proof}
        We have:
        \begin{equation}
            \begin{pmatrix}
                \Re(\momVal^{(l+1)})\\
                \Im(\momVal^{(l+1)})\\
                \bothParam^{(l+1)}
            \end{pmatrix}
            = \dynamics 
            \begin{pmatrix}
                \Re(\momVal^{(l)})\\
                \Im(\momVal^{(l)})\\
                \bothParam^{(l)}
            \end{pmatrix}
            + \B{q}^{(l)}
        \end{equation}
        By telescoping the recurrence for the $l^{th}$ layer: 
        \begin{equation}
            \begin{pmatrix}
                \Re(\momVal^{(l)})\\
                \Im(\momVal^{(l)})\\
                \bothParam^{(l)}
            \end{pmatrix}
            = \dynamics^l
            \begin{pmatrix}
                \Re(\momVal^{(0)})\\
                \Im(\momVal^{(0)})\\
                \bothParam^{(0)}
            \end{pmatrix}
            + \displaystyle \sum_{i=0}^{l-1} \dynamics^i \B{q}^{(i)}
        \end{equation}
        
        We can compare $\momVal^l$ and $\displaystyle \sum_{i=0}^{l-1} \dynamics^i \B{q}^{(i)}$ with the values $\momVal^*$  and $\B{q}^*$ they converge to which exists if $\dynamics$ is contractive.
        We do the same with $\bothParam$.
        Because $\momVal^* = \dynamics \momVal^* + \B{q}^* = \dynamics^l \momVal^* + \displaystyle \sum_{i=1}^\infty \dynamics^i \B{q}^{(i)}$:
        \begin{equation}
            \begin{pmatrix}
                \Re(\momVal^{(l)}) - \Re(\momVal^{*})\\
                \Im(\momVal^{(l)}) - \Im(\momVal^{*})\\
                \bothParam^{(l)} - \bothParam^{*}
            \end{pmatrix}
            = \dynamics^l
            \begin{pmatrix}
                \Re(\momVal^{(0)})  - \Re(\momVal^{*})\\
                \Im(\momVal^{(0)})  - \Im(\momVal^{*})\\
                \bothParam^{(0)} - \bothParam^{*}
            \end{pmatrix}
            + \displaystyle \sum_{i=0}^{l-1} \dynamics^l \B{q}^{(i)} - \displaystyle \sum_{l=1}^\infty \dynamics^i \B{q}^{(i)}
        \end{equation}
        By taking norms:
        \begin{align}
            \left\Vert
                \begin{pmatrix}
                    \Re(\momVal^{(l)}) - \Re(\momVal^{*})\\
                    \Im(\momVal^{(l)}) - \Im(\momVal^{*})\\
                    \bothParam^{(l)} - \bothParam^{*}
                \end{pmatrix}
            \right\Vert
            &= 
            \left\Vert
                \dynamics^l
                \begin{pmatrix}
                    \Re(\momVal^{(0)})  - \Re(\momVal^{*})\\
                    \Im(\momVal^{(0)})  - \Im(\momVal^{*})\\
                    \bothParam^{(0)} - \bothParam^{*}
                \end{pmatrix}
               - \displaystyle \sum_{i=l}^\infty \dynamics^i \B{q}^{(i)}
            \right\Vert
            \\
            \implies
            \left\Vert
                \begin{pmatrix}
                    \Re(\momVal^{(l)}) - \Re(\momVal^{*})\\
                    \Im(\momVal^{(l)}) - \Im(\momVal^{*})\\
                    \bothParam^{(l)} - \bothParam^{*}
                \end{pmatrix}
            \right\Vert
            &\leq
            \left\Vert
                \dynamics^l
            \right\Vert
            \left\Vert
                \begin{pmatrix}
                    \Re(\momVal^{(0)})  - \Re(\momVal^{*})\\
                    \Im(\momVal^{(0)})  - \Im(\momVal^{*})\\
                    \bothParam^{(0)} - \bothParam^{*}
                \end{pmatrix}
            \right\Vert
            + \left\Vert \displaystyle \sum_{i=l}^\infty \dynamics^i \B{q}^{(i)} \right\Vert  \\
            & \leq \left\Vert
                \dynamics^l
            \right\Vert
            \left\Vert
                \begin{pmatrix}
                    \Re(\momVal^{(0)})  - \Re(\momVal^{*})\\
                    \Im(\momVal^{(0)})  - \Im(\momVal^{*})\\
                    \bothParam^{(0)} - \bothParam^{*}
                \end{pmatrix}
            \right\Vert 
            + \left\Vert \displaystyle \sum_{i=l}^\infty \dynamics^i \right\Vert B
        \end{align}
        With Lemma 11 from \citep{foucart2012}, we have there exists a matrix norm $\forall \epsilon > 0$ such that:
        \begin{equation}\label{eq:foucart}
            \| \dynamics^l \| \leq D \left(\spectralRadius \left(\dynamics \right) + \epsilon\right)^l
        \end{equation}
        %
        We also have 
        \begin{equation}
        \label{eq:norm_bound}
            0 < \left\Vert \displaystyle \sum_{i=0}^\infty \dynamics^i \right\Vert < C
        \end{equation} 
        if $\dynamics$ is contractive.
        Combining Equation \eqref{eq:foucart} and Equation \eqref{eq:norm_bound} we have:
        \begin{equation}
            \left\Vert
                \begin{pmatrix}
                    \Re(\momVal^{(l)}) - \Re(\momVal^{*})\\
                    \Im(\momVal^{(l)}) - \Im(\momVal^{*})\\
                    \bothParam^{(l)} - \bothParam^{*}
                \end{pmatrix}
            \right\Vert
            \leq
            D \left(\spectralRadius \left(\dynamics \right) + \epsilon\right)^l
            \left\Vert
                \begin{pmatrix}
                    \Re(\momVal^{(0)})  - \Re(\momVal^{*})\\
                    \Im(\momVal^{(0)})  - \Im(\momVal^{*})\\
                    \bothParam^{(0)} - \bothParam^{*}
                \end{pmatrix}
            \right\Vert + BC
        \end{equation}
        So, we have:
        \begin{equation}
            \left\Vert
                \begin{pmatrix}
                    \Re(\momVal^{(l)}) - \Re(\momVal^{*})\\
                    \Im(\momVal^{(l)}) - \Im(\momVal^{*})\\
                    \bothParam^{(l)} - \bothParam^{*}
                \end{pmatrix}
            \right\Vert
            = \mathcal{O}((\spectralRadius(\dynamics) + \epsilon)^l)
        \end{equation}
        Thus, we converge linearly with a rate of $\mathcal{O}(\spectralRadius(\dynamics) + \epsilon)$.
    \end{proof}
\subsection{Propostion~\ref{prop:theorem_existence} Proof Sketch}\label{pf:prop}
\textbf{Proposition 3.5}(Convergence rate of NAMEx-Momentum). \textit{There exist $\lr \in \mathbb{R}^+, \momCoeff \in \mathbb{C}$ so Algorithm~\ref{alg:simultaneous_complex} converges for Expert-Propagation NAMEx Momentum.}
\begin{proof}
We want to show that we can select $\gamma > 0$ and $\beta \in \mathbb{C}$ so that $\dynamics$ is contractive. That is, the spectral radius of $\dynamics$ is less than 1. Recall that,
\begin{equation} 
    \dynamics = \begin{bmatrix}
        \Re(\momCoeff) \identity & -\Im(\momCoeff) \identity & -\suma \identity\\
        \Im(\momCoeff) \identity & \Re(\momCoeff) \identity & 0\\
        \Re(\lr \momCoeff) \identity & -\Im(\lr \momCoeff) \identity & \identity - \lr\suma\identity\\
        \end{bmatrix}
\end{equation}

Set $\beta = r + ui$, we have:
\begin{equation}
    \det (\dynamics - x\identity) 
    = \det \left(\begin{bmatrix}
        &r-x &-u &-\suma \\
        &u   &r - x   &0 \\
        &\gamma r & -u   & 1 - \gamma \suma -x  
    \end{bmatrix} \otimes \identity\right)
    = \det 
    \left(\begin{bmatrix}
        &r-x &-u &-\suma \\
        &u   &r - x   &0 \\
        &\gamma r & -u   & 1 - \gamma \suma -x  
    \end{bmatrix} \right)^d
\end{equation}

\begin{align}
\begin{vmatrix}
r - x & -u & -\suma \\
u & r - x & 0 \\
\gamma r & -u & 1 - \gamma \suma - x
\end{vmatrix} 
& = r^2 + \suma u^2 - \gamma \suma u^2 + u^2 - x^3 - \gamma \suma x^2 + 2r x^2 + x^2 - r^2 x + \gamma \suma r x - 2r x - u^2 x \\
&= -x^3 + (-\gamma \suma + 2r + 1) x^2 + (-r^2 + \gamma \suma r - 2r - u^2) x + r^2 + (\suma - \gamma \suma + 1) u^2
\end{align}

We can further simplify this by choosing $\gamma = \dfrac{\hat\gamma}{\suma}$ to get the following polynomial,
\begin{equation}
    P(x) = -x^3 + (1 + 2r - \hat\gamma)x^2 + (-r^2 + \hat\gamma r - 2r - u^2)x + \big(r^2 + (1 + \suma - \hat\gamma)u^2\big)
\end{equation}
Using Fujiwara's bound \citep{fujiwaraber1916} we can determine one condition for $\rho(\dynamics) < 1$, that is
\begin{align}
|x| \leq 2 \max \left\{ |1 + 2r - \hat \gamma|, \sqrt{| -r^2 + \hat \gamma r - 2r - u^2|}, \sqrt[3]{\left\vert\dfrac{r^2 + (1 + \suma - \hat \gamma)u^2}{2}\right\vert} \right\} < 1
\end{align}
\begin{align}
\Rightarrow&
\label{eq:ineq}
  \begin{cases}
    \dfrac{-1}{2} < 1 + 2r - \hat \gamma < \dfrac{1}{2} \\[0.5cm]
    \dfrac{-1}{4} < r^2  - \hat \gamma r + 2r + u^2 < \dfrac{1}{4} \\[0.5cm]
    -\dfrac{1}{4} < r^2 + (1 + \suma - \hat \gamma)u^2 < \dfrac{1}{4}
\end{cases}  \\[0.5cm]
\Leftrightarrow&
\begin{cases}
    \dfrac{-1}{4} + \dfrac{\hat\gamma}{2} - \dfrac{1}{2} < r < \dfrac{1}{4} + \dfrac{\hat\gamma}{2} - \dfrac{1}{2} \\[0.5cm]
    \dfrac{-1}{4}-r^2 - \hat \gamma r - 2r  < u^2 < \dfrac{1}{4}-r^2 - \hat \gamma r - 2r \\[0.5cm]
    \dfrac{-1-4r^2}{1 + \suma - \hat \gamma} < u^2 < \dfrac{1-4r^2}{1 + \suma - \hat \gamma}
\end{cases}
\end{align}
We can consider the case when $\suma = 0.5$ and $\gamma = 2$. Figure \ref{fig: ineq} shows the region of $(r,u)$ that satisfies inequality system \eqref{eq:ineq}.

\begin{figure}
    \centering
    \includegraphics[width=0.8\linewidth]{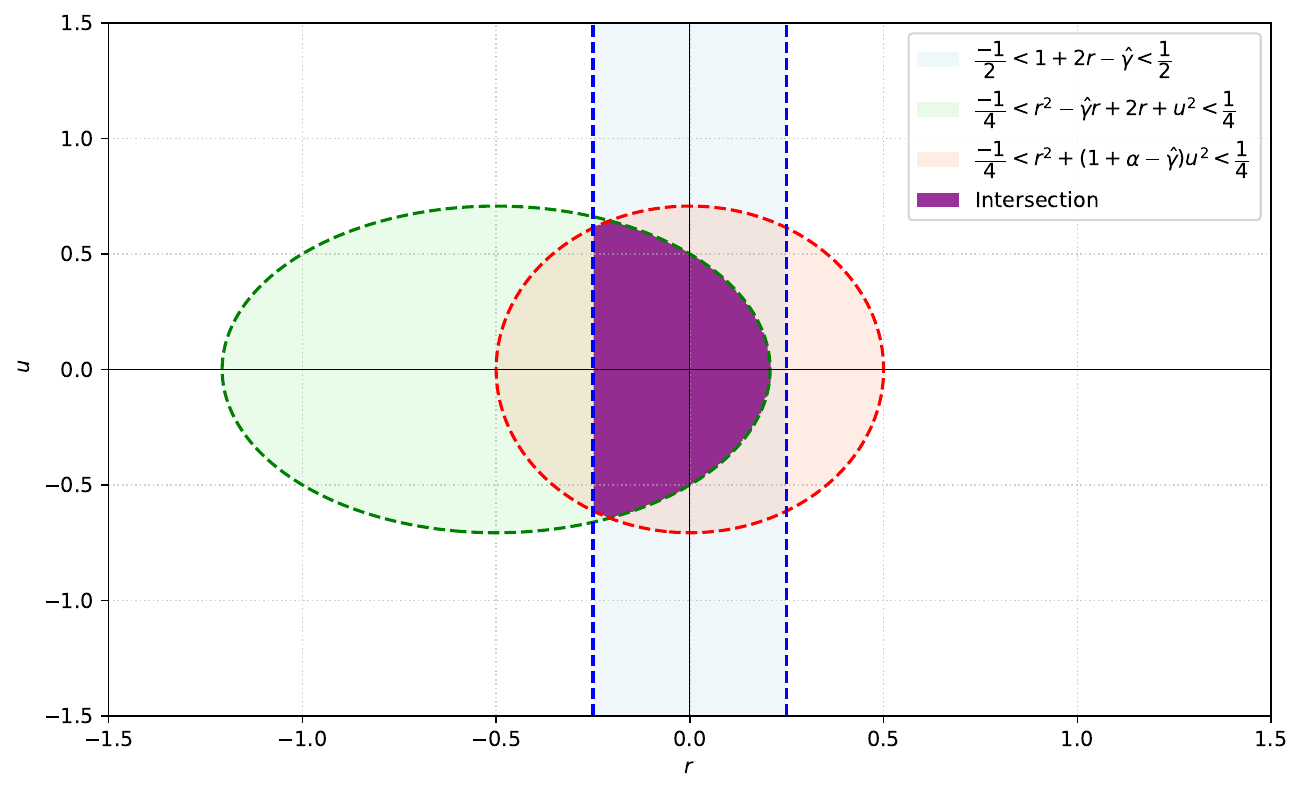}
    \vspace{-0.8em}
    \caption{Graph of system of inequality \eqref{eq:ineq} when $\suma = 0.5$ and $\gamma = 2$.}
    \label{fig: ineq}
\end{figure}
\end{proof}

\section{Additional Details on Datasets}
\label{app:dataset-model-details}

This section outlines the datasets and evaluation metrics employed in the experiments discussed in Section \ref{sec:experiment}.
\subsection{WikiText-103 Language Modeling}
\textbf{The WikiText-103} dataset contains a collection of Wikipedia articles designed to capture long-range contextual dependencies. It includes a training set with 28,475 articles, amounting to around 103 million words. The validation and test sets consist of 217,646 and 245,569 words, respectively, distributed across 60 articles per set.

\textbf{Model and baselines} We use the small and medium size Transformer as the baseline SMoE models. Our implementation is based on the codebase developed by \citep{pham2024competesmoeeffectivetraining} and \citep{teo2024momentumsmoe}. All model variants\textemdash SMoE, CAMEx, and NAMEx\textemdash employ 16 experts per layer, with SMoE utilizing top-1 (k=1) expert selection for each input. The models share a unified sparse routing mechanism, consisting of a linear layer to process the input, followed by Top-K and Softmax functions. Training is performed over 60 epochs for small models and 80 epochs for medium and large SMoE models.

\subsection{GLUE Text Classification}
The GLUE benchmark includes \textbf{SST-2} \citep{socher2013recursive} for sentiment analysis, \textbf{MRPC} \citep{dolan2005mrpc} for paraphrase detection and sentence similarity, \textbf{CoLA} \citep{warstadt2019cola} for evaluating grammatical acceptability, \textbf{STS-B} \citep{cer2017stsb} for sentence similarity measurement, \textbf{RTE} \citep{dagan2006rte} for logical reasoning, \textbf{QNLI} \citep{wang2019glue} for question-answer classification, and \textbf{MNLI} \citep{williams2018mnli} for assessing entailment between sentence pairs.

\textbf{Model and baselines}
We scale up T5 \citep{raffel2020t5} using SMoE upcycling \citep{komatsuzaki2023sparse}. For each task, we conduct a comprehensive hyperparameter search, exploring batch sizes \{$8$, $16$, $32$, $64$\} and learning rates \{$3e{-4}$, $1e{-4}$, $3e{-5}$, $1e{-5}$\} to identify the optimal fine-tuned configuration.

\subsection{ImageNet-1k Image Classification}
\textbf{ImageNet-1k}, introduced by \citep{deng2009imagenet}, is a widely used benchmark dataset comprising 1.28 million images for training and 50,000 images for validation across 1,000 categories. Performance is evaluated using top-1 and top-5 accuracy metrics.

For robustness evaluation, we utilize several specialized subsets.
\textbf{ImageNet-A}~\citep{hendrycks2021natural} focuses on 200 challenging classes from ImageNet-1k, specifically curated to fool classifiers, highlighting their vulnerability to real-world adversarial examples. \textbf{ImageNet-O}~\citep{hendrycks2021natural} contains out-of-distribution (OOD) samples derived from ImageNet-22k, carefully selected as instances that a ResNet-50 model misclassifies with high confidence. The primary evaluation metric for ImageNet-O is the area under the precision-recall curve (AUPR). Lastly, \textbf{ImageNet-R}~\citep{hendrycks2021many} consists of 30,000 artistic renditions representing 200 classes from ImageNet-1k, designed to assess model robustness to non-standard visual representations.

\textbf{Model and baselines}
For each MoE layer, we use \cref{alg:simultaneous_complex} to merge all experts into a base expert, except in the first MoE layer, where a base expert is initialized instead. Training configurations follow Swin-MoE \citep{liu2021Swin}, and the code is publicly available on \href{https://github.com/microsoft/Swin-Transformer/}{https://github.com/microsoft/Swin-Transformer/}. For NAMEx variants, we start with checkpoints pretrained on ImageNet-22k and fine-tune them on ImageNet-1k for 30 epochs.

\section{More experimental details}
This section provides additional details on the experimental setup, including model configurations, dataset processing, and training strategies used in our evaluation.

\subsection{WikiText-103 Language Modeling}
We follow the setup from \citep{pham2024competesmoeeffectivetraining} and \citep{teo2024momentumsmoe}, using both small and medium-scale Transformer architectures with 16 experts per layer. All variants (SMoE, CAMEx, EP-CAMEx, NAMEx) use Top-1 routing and share the same sparse gating mechanism. Training is conducted for 60 epochs (small) and 80 epochs (medium) with AdamW optimizer and cosine learning rate scheduling.

\subsection{GLUE Benchmark Fine-Tuning}
For text classification, we fine-tune T5-base models upcycled with SMoE layers. We conduct grid searches over batch sizes {8, 16, 32, 64} and learning rates {$3 \times 10^{-5}$, $1 \times 10^{-4}$, $3 \times 10^{-4}$}. Each result is averaged over five seeds to ensure statistical stability. All SMoE variants employ 8 experts per layer and share the same routing logic.

\subsection{ImageNet-1k and Corrupted Variants}
For vision experiments, we fine-tune Swin-MoE-Small on ImageNet-1k for 30 epochs using batch size 96 and learning rate $1 \times 10^{-4}$. NAMEx variants initialize $\mathbf{E}_m$ in the first MoE layer and perform merging across all others via Algorithm~\ref{alg:namex} or Algorithm~\ref{alg:simultaneous_complex}. For robustness, we evaluate zero-shot generalization on ImageNet-A, ImageNet-O, and ImageNet-R. All reported results are averaged over three runs.

\subsection{Implementation and Infrastructure}
Experiments are implemented in PyTorch and trained on 4–8 A100 GPUs depending on model size. We use automatic mixed precision (AMP) for memory efficiency. All hyperparameters, data augmentations, and merging schedules are described in Appendix~\ref{app:dataset-model-details}.

\section{Additional Experimental Results}
\label{app:add-exp}
\subsection{Zero-shot and Finetuning on Qwen1.5-MoE (14B parameters)}
\label{app:qwen}
To assess scalability, we integrate NAMEx-Full into Qwen1.5-MoE (14B) and evaluate it on three challenging benchmarks: MMLU \citep{hendrycks2021measuring}, GSM8K \citep{cobbe2021trainingverifierssolvemath}, and ARC \citep{clark2018thinksolvedquestionanswering} in both zero-shot and fine-tuned settings. In the fine-tuned setting, all models are fine-tuned on the SmolTalk~\citep{allal2025smollm} dataset before evaluation. As reported in Tables~\ref{tab:qwen_zeroshot} and~\ref{tab:qwen_finetune}, NAMEx-Full consistently outperforms both the baseline and EP-CAMEx across routing schemes and tasks, highlighting its robustness, scalability, and architectural generality.

\begin{table}[h]
  \centering
  \caption{Zero\,-\,shot results for \qwen{} variants.}
  \vspace{-0.8em}
  \setlength{\tabcolsep}{6pt}
  \resizebox{0.6\linewidth}{!}{
  \begin{tabular}{l l r r r}
    \toprule
   \textbf{Routing Strategy} & \textbf{Model} & \textbf{MMLU} & \textbf{GSM8K} & \textbf{ARC} \\
    \midrule
    Linear & \qwen{} & 61.28 & 60.12 & 50.77 \\
            & \epcamex{} & 61.54 & 60.23 & 50.83 \\
            \rowcolor{blue!10}
            & \textbf{\namex{}\,-\,Full} & \textbf{61.87} & \textbf{60.55} & \textbf{50.95} \\
            \midrule
    Cosine  & \qwen{} & 61.10 & 59.88 & 50.60 \\
            & \epcamex{} & 61.40 & 60.00 & 50.68 \\
            \rowcolor{blue!10}
            & \textbf{\namex{}\,-\,Full} & \textbf{61.85} & \textbf{60.52} & \textbf{50.93} \\
            \midrule
    Stable\,-\,MoE & \qwen{} & 61.35 & 60.22 & 50.81 \\
            & \epcamex{} & 61.60 & 60.35 & 50.89 \\
            \rowcolor{blue!10}
            & \textbf{\namex{}\,-\,Full} & \textbf{61.90} & \textbf{60.60} & \textbf{50.96} \\
    \bottomrule
  \end{tabular}
}
  \label{tab:qwen_zeroshot}
\end{table}

\begin{table}[h!]
\centering
\caption{Results after fine-tuning on SmolTalk.}
\vspace{-0.8em}
\resizebox{0.6\linewidth}{!}{
\begin{tabular}{llccc}
\toprule
\textbf{Routing Strategy} & \textbf{Model} & \textbf{MMLU} & \textbf{GSM8K} & \textbf{ARC} \\
\toprule
\multirow{3}{*}{Linear} 
    & Qwen1.5-MoE   & 61.50 & 60.52 & 51.12 \\
    & EP-CAMEx      & 61.74 & 60.63 & 51.23 \\
    \rowcolor{blue!10}
    & \textbf{NAMEx-Full}    & \textbf{62.10} & \textbf{61.00} & \textbf{51.35} \\
\midrule
\multirow{3}{*}{Cosine} 
    & Qwen1.5-MoE   & 61.30 & 60.28 & 50.95 \\
    & EP-CAMEx      & 61.60 & 60.50 & 51.10 \\
    \rowcolor{blue!10}
    & \textbf{NAMEx-Full}    & \textbf{62.05} & \textbf{60.95} & \textbf{51.30} \\
\midrule
\multirow{3}{*}{Stable-MoE Routing} 
    & Qwen1.5-MoE   & 61.60 & 60.65 & 51.20 \\
    & EP-CAMEx      & 61.85 & 60.80 & 51.30 \\
    \rowcolor{blue!10}
    & \textbf{NAMEx-Full}    & \textbf{62.15} & \textbf{61.10} & \textbf{51.45} \\
\bottomrule
\end{tabular}
}
\label{tab:qwen_finetune}
\end{table}

\section{Additional Empirical Analysis}
\label{appx:add-emp-analysis}

\subsection{Overhead and Scalability}
\label{sec:overhead}
\begin{table}[h]
  \centering
  \caption{Impact of NBS update frequency ($\Delta \ell$). Per\,-\,batch wall\,-\,clock.}
  \vspace{-0.8em}
  \setlength{\tabcolsep}{6pt}
  \resizebox{0.8\linewidth}{!}{
  \begin{tabular}{c C{1.2cm} C{1.2cm} C{1.2cm} C{1.2cm} C{2.5cm} C{1.9cm}}
    \toprule
    $\Delta \ell$ & SST\,-\,2 & MRPC & STS\,-\,B & RTE & Mean Batch (s) & \% NBS Time \\
    \midrule
    1 & 94.88 & 92.85 & 90.32 & 77.26 & 4.70 & 85.96\% \\
    2 & 94.95 & 92.38 & 90.37 & 76.89 & 2.29 & 71.18\% \\
    5 & 95.18 & 92.09 & 90.13 & 77.98 & 1.14 & 42.11\% \\
    $L$ (first layer) & 94.46 & 92.01 & 90.12 & 75.09 & \textbf{0.69} & \textbf{4.35\%} \\
    \bottomrule
  \end{tabular}
}
  \label{tab:delta}
\end{table}

\begin{table}[h]
  \centering
  \caption{Training compute and throughput. Inference is unchanged relative to baselines.}
  \vspace{-0.8em}
  \setlength{\tabcolsep}{6pt}
  \resizebox{0.7\linewidth}{!}{
  \begin{NiceTabular}{l r r r}
    \toprule
    Model & Train TFLOPs & Infer TFLOPs & Train Throughput (tok/s) \\
    \midrule
    \smoe{} & 13.95 & 4.65 & 22,236 \\
\RowStyle{}\textbf{\smoe{}\,(Top-2)} & 18.32 & 7.44 & 17,898 \\
    \RowStyle{} SMEAR & 13.95 & 4.65 & 22,236 \\
    \camex{} & 14.30 & 4.65 & 21,982 \\
    \epcamex{} & 14.25 & 4.65 & 21,982 \\
    \namex{} & 14.25 & 4.65 & 21,995 \\
    \textbf{\namex{}\,-\,Full} & \textbf{14.25} & \textbf{4.65} & \textbf{21,897} \\
    EP\,-\,CAMEx\,-\,Mom & 14.25 & 4.65 & 21,872 \\
    \textbf{\namex{}\,-\,Full\,-\,Mom} & \textbf{14.25} & \textbf{4.65} & \textbf{21,783} \\
    \bottomrule
  \end{NiceTabular}
}
  \label{tab:throughput}
\end{table}

In Tab. \ref{tab:delta}, we provide a detailed runtime cost analysis, “Mean Batch Runtime (sec)” includes the entire forward/backward pass and NBS step. "\% NBS Occupied" isolates the share of time spent solving NBS.

 In Tab. \ref{tab:throughput}, we provide a comparison of training TFLOPs, inference TFLOPs, and training throughput across baselines and our proposed variants in Table 2 below. Notably, NAMEx-Full achieves competitive throughput (21,897 tokens/sec), closely matching CAMEx and EP-CAMEx despite the added complexity of solving the Nash system. \tmn{While NAMEx and NAMEx-Mom show slightly lower throughput due to the large number of NBS iterations (20 iters), NAMEx-Full-Mom restores much of the efficiency by reducing the number of NBS iterations (2 iters per layer).} Currently, the main overhead arises from transferring data to the CPU for solving the Nash Bargaining update step. However, this implementation detail is orthogonal to the algorithm itself and can be optimized via GPU-native solvers or batching strategies. Overall, the results demonstrate that NAMEx introduces minimal overhead and remains practical for large-scale MoE training, supporting its applicability to future deployments involving more experts per layer.

\subsection{Convergence Analysis}
\label{sec:convergence}

\begin{figure}[h]
  \centering
  \includegraphics[width=0.8\linewidth]{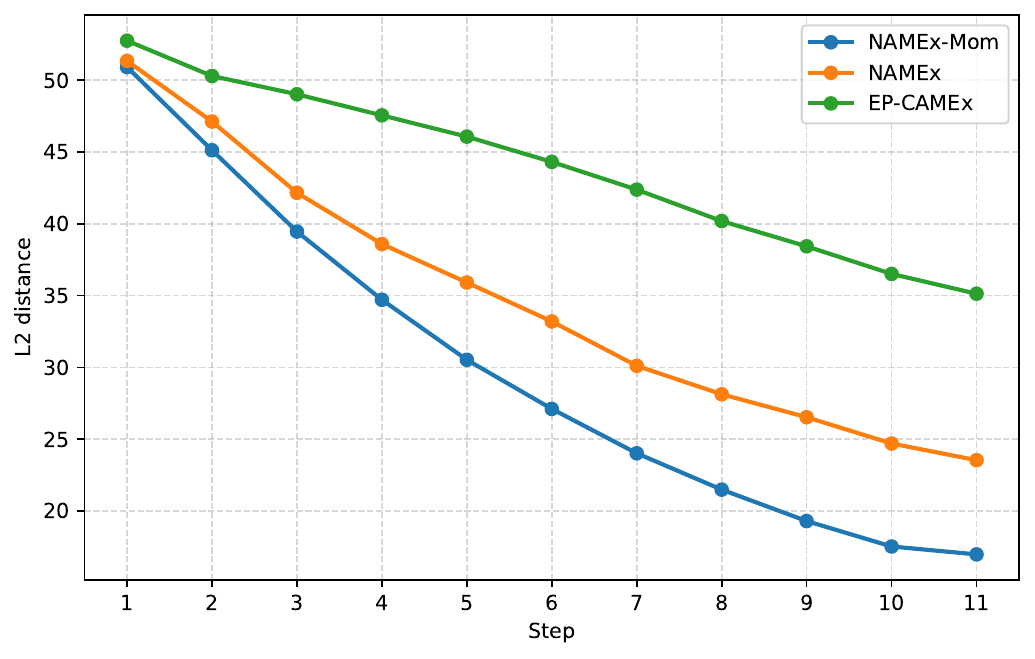}
  \vspace{-0.8em}
  \caption{L2 distance between expert updates across training steps (T5-Base, 12 MoE layers). Lower values indicate better stability. The figure shows that NAMEx converges faster and more stably than EP-CAMEx.}
  \label{fig:l2_distance_across_steps}
\end{figure}

To validate the motivation for complex momentum, we provide empirical convergence analysis in Fig.~\ref{fig:l2_distance_across_steps}, which tracks the L2 distance between updates of base experts across training steps (T5-Base, 12 MoE layers). As illustrated, NAMEx--with or without momentum--shows a noticeably steeper decline in update distances, indicating faster convergence and more stable expert updates compared to EP-CAMEx. This directly supports our hypothesis that complex momentum enhances convergence stability and efficiency during expert merging.

\subsection{Swin\,-\,MoE\,-\,S: 90\,-\,Epoch Fine\,-\,Tuning}
\label{sec:swin}
\begin{table}[h]
  \centering
  \caption{ImageNet\,-\,1K Top\,-\,1 from epochs 86--90. Our best final performance is \textbf{85.046\%}.}
  \vspace{-0.8em}
  \setlength{\tabcolsep}{4pt}
  \resizebox{1.0\linewidth}{!}{
  \begin{tabular}{r r r r r r r r}
    \toprule
    Epoch & \namex{}\,-\,Mom & \namex{} & EP\,-\,\camex{}\,-\,Mom & \epcamex{} & \namex{}\,-\,Full & \namex{}\,-\,Full\,-\,Mom & Swin\,-\,MoE \\
    \midrule
     86 & 84.466 & 84.264 & 83.862 & 82.238 & 84.722 & \textbf{85.022} & 83.435 \\
     87 & 84.504 & 84.252 & 83.868 & 82.286 & 84.728 & \textbf{85.028} & 83.400 \\
     88 & 84.502 & 84.240 & 83.854 & 82.234 & 84.734 & \textbf{85.034} & 83.379 \\
     89 & 84.540 & 84.228 & 83.860 & 82.282 & 84.740 & \textbf{85.040} & 83.413 \\
     90 & 84.518 & 84.216 & 83.806 & 82.230 & 84.746 & \textbf{85.046} & 83.415 \\
    \bottomrule
  \end{tabular}
}
  \label{tab:swin90}
\end{table}

Tab. \ref{tab:swin90} summarizes the top-1 accuracy on ImageNet-1K from epochs 86 to 90, these results confirm that our models continue to improve with more training, and that Namex-full-Mom reaches a final top-1 accuracy of 85.046\%. This demonstrates both competitive final performance and strong convergence behavior. Note that due to the difference in number of GPU being used, it seems that we could not reproduce the 84.5\% result as reported by the official repo of Swin-MoE.

\subsection{Other Results}
\label{sec:other_results}
\cref{fig:wiki-line} presents the evaluation perplexity on WikiText-103 during training. NAMEx-Mom achieves the lowest validation and test perplexities in both small- and medium-scale pre-training, outperforming SMoE and CAMEx-based methods. Across all settings, Nash variants, including NAMEx and NAMEx-Mom, consistently surpass their counterparts, demonstrating the effectiveness of Nash bargaining and momentum mechanisms.

\cref{fig:switch-heatmap} and \cref{fig:swin-heatmap} visualize the cosine similarity between experts ouput at all SMoE layers indicating a complex dynamic of how experts at progressive layers interact with each other. \emph{This observation suggests that the behavior of experts cannot be captured optimally by using simple averaging as of the previous work}. Instead, a more effective strategy for determining the merging coefficients should account for the experts' dynamics at each specific layer.
\label{app:extra_emp_analysis}
\begin{figure}[t]
    \centering
    \includegraphics[width=0.6\linewidth]{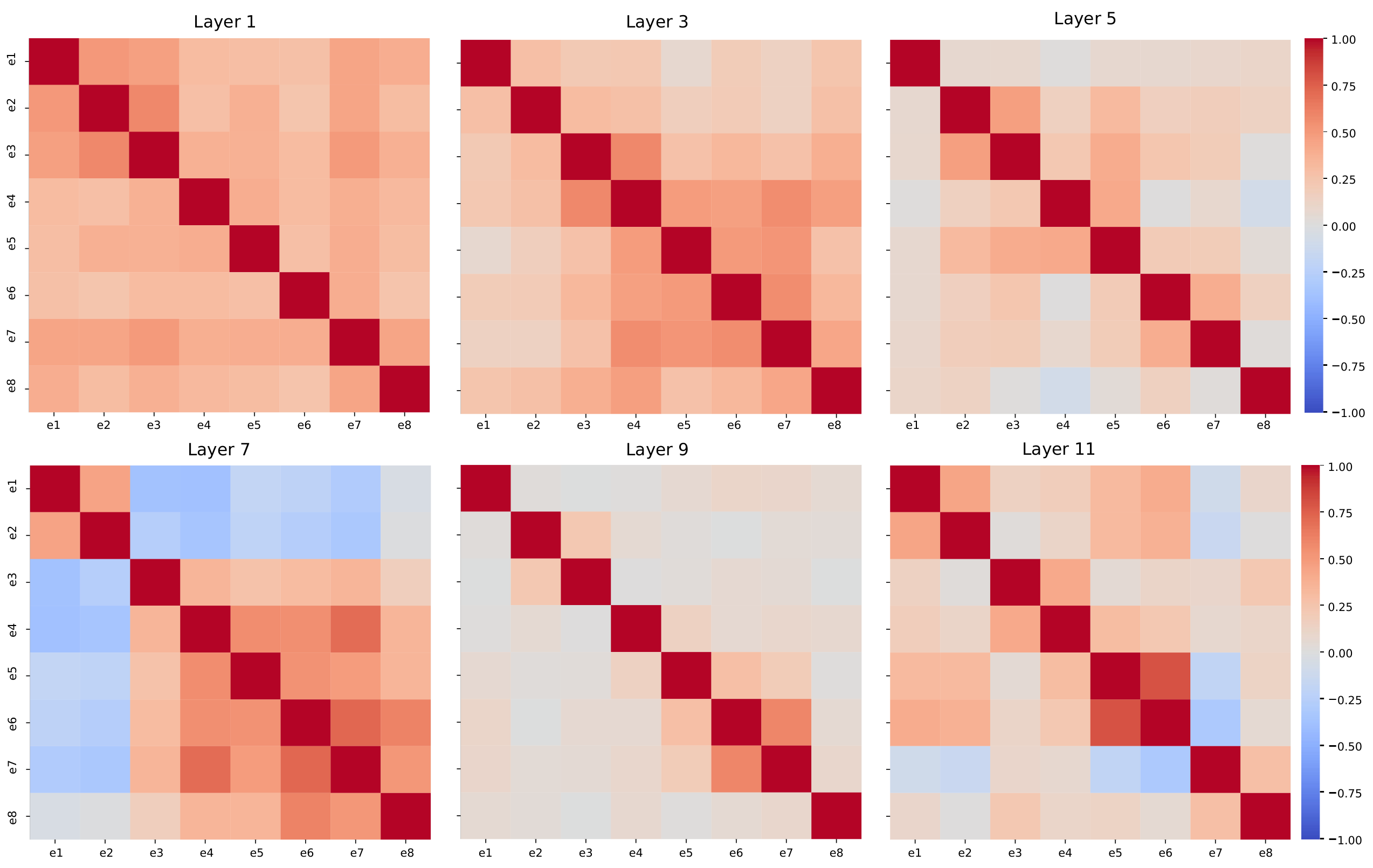}
    \vspace{-0.8em}
    \caption{Cosine similarity between expert outputs in Switch-Transformers.}
    \label{fig:switch-heatmap}
\end{figure}

\begin{figure}[t]
    \centering
    \includegraphics[width=0.6\linewidth]{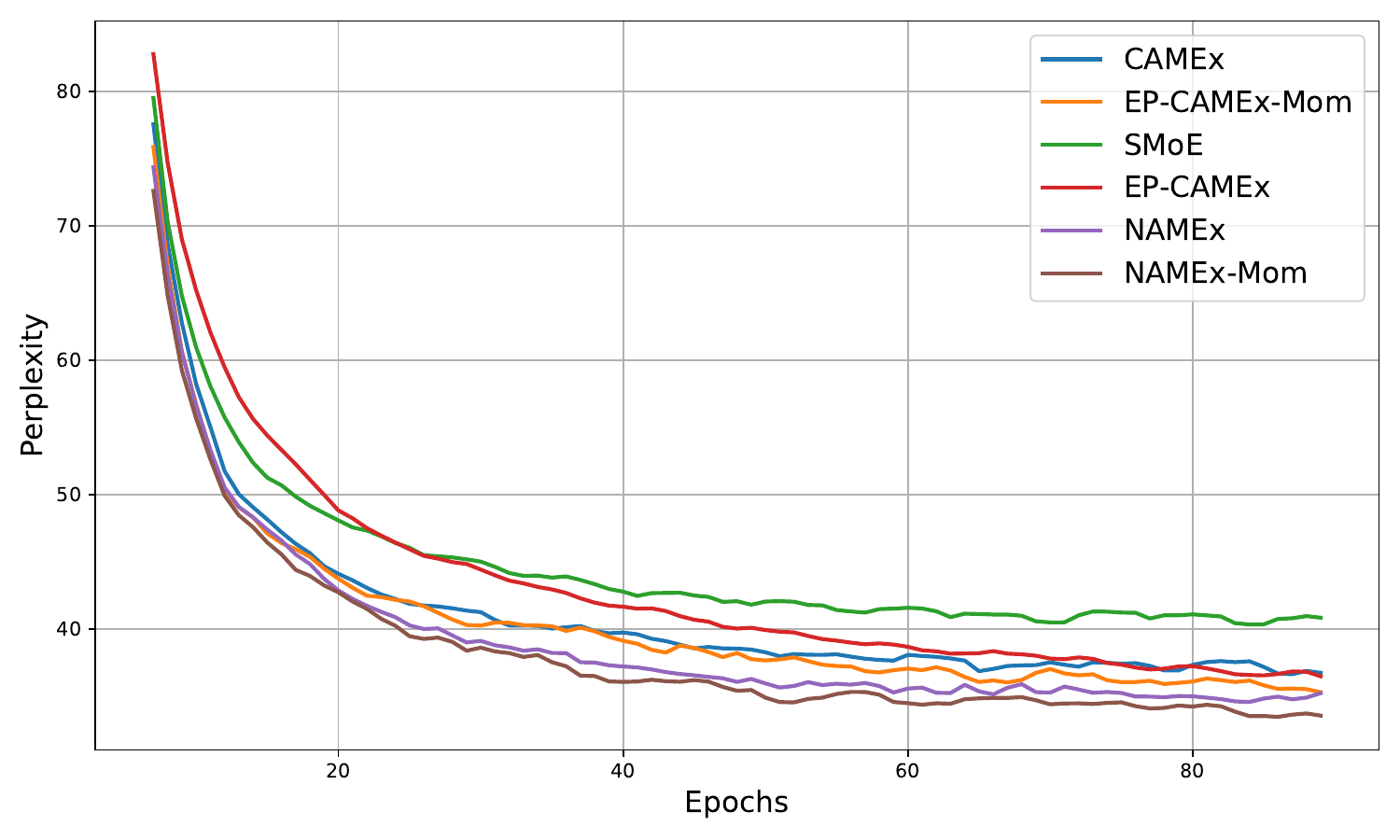}
    \vspace{-0.8em}
    \caption{5-Period Moving Average of Perplexity of different Transformers-medium variants on WikiText-103}
    \label{fig:wiki-line}
\end{figure}
\begin{figure}[t]
    \centering
    \includegraphics[width=0.45\linewidth]{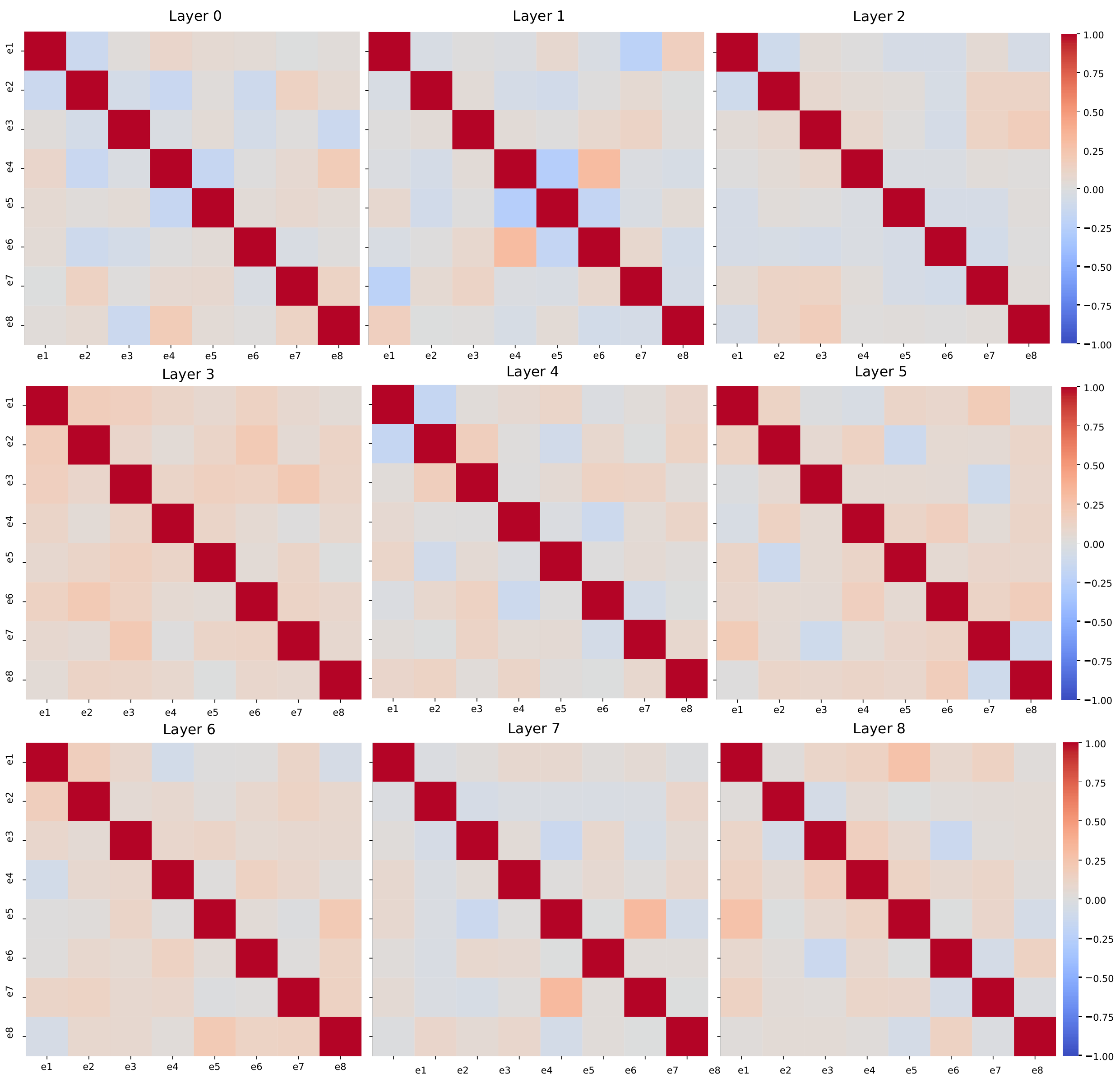}
    \vspace{-0.8em}
    \caption{Cosine similarity between expert outputs in Swin-MoE.}
    \label{fig:swin-heatmap}
\vspace{-0.15in}
\end{figure}
\begin{figure}[t]
    \centering
    \includegraphics[width=0.45\linewidth]{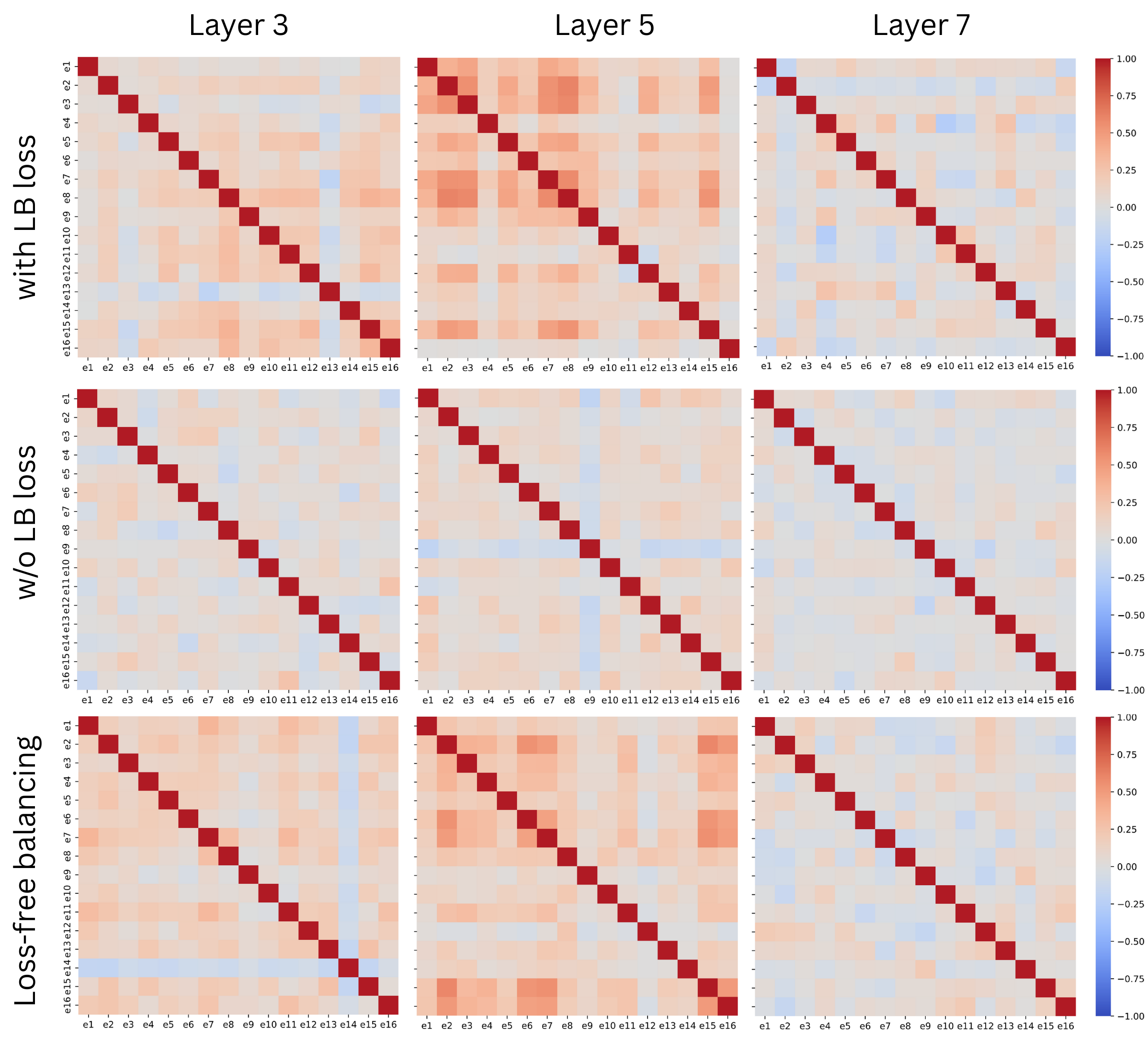}
    \caption{We compared expert interaction patterns under three settings: with Load Balancing loss, without Load Balancing loss, and loss-free balancing (in the sense of \citep{wang2025auxiliarylossfree}).}
    \label{fig:lb-and-nlb}
    \vspace{-0.1in}
\end{figure}
\begin{figure}[t]
    \centering
    \includegraphics[width=0.9\textwidth]{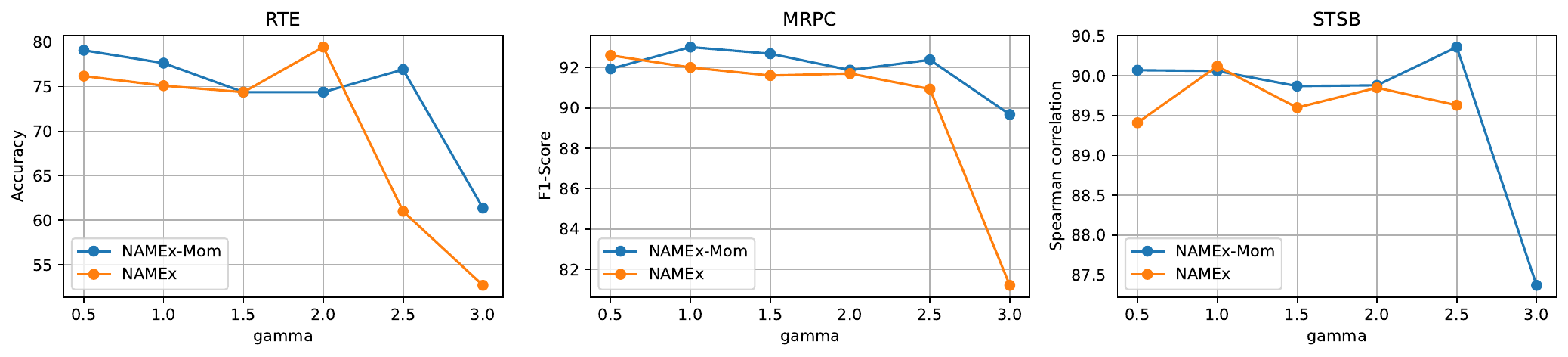}
    \caption{Impact of different values of step-size $\gamma$ on NAMEx and NAMEx-Mom performance. Overall, the optimal setting lies within the range $[0.5,2]$. For $\gamma > 2$, the performance drops significantly, which may indicate an overshooting phenomenon.}
    \label{fig:gamma}
\end{figure}
As shown in Figure \ref{fig:lb-and-nlb}, the cooperative/competitive dynamics (as reflected in the off-diagonal correlation values) are much more visible when expert load is balanced--either through Load Balancing loss or loss-free mechanisms. In contrast, when training without Load Balancing loss, many experts appear less specialized, and the interaction patterns become weaker and less structured.

One hypothesis is that, without balanced token routing, some experts may be underused or even become inactive, which diminishes the emergence of meaningful cooperative or competitive behavior. Therefore, balanced expert load is not only important for preventing dead experts but also plays a crucial role in making such dynamics observable and analyzable.
\begin{figure}
    \centering
    \includegraphics[width=0.8\linewidth]{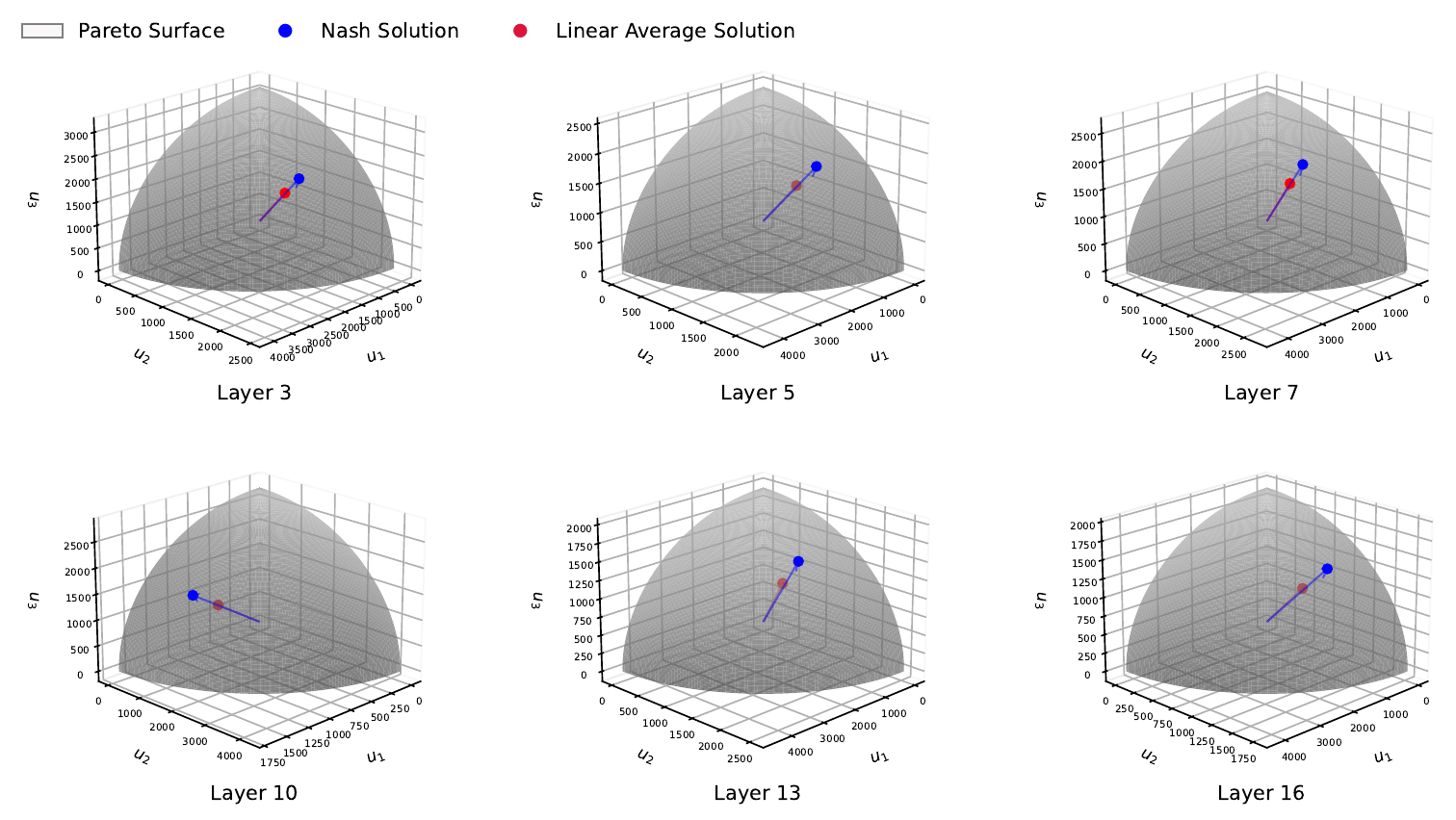}
    \vspace{-0.8em}
    \caption{Visualization of expert utility trade-offs accross multiple SMoE layers. Each subplot corresponds to a different layer, with arrows indicating merging directions in the 3D utility space. The blue arrow represents the Nash Bargaining solution (NAMEx), while the red arrow denotes the average-based merging direction. The hemispherical surface is uniformly sampled to illustrate the feasible utility region. Across layers, the Nash direction consistently steers toward more balanced expert cooperation compared to the naive average.}
    \label{fig:pareto-plot}
\end{figure}

\begin{table}[t!]
    \centering
    \caption{Pretraining and zero-shot results for NAMEx vs. ACMoE Top-1/Top-2 on ImageNet-1k and corrupted variants.}
    \label{tab:imagenet-acmoe-results}
    {
    \resizebox{0.8\linewidth}{!}{
    \begin{tabular}{lcccccc}
        \toprule
        \textbf{Model} & \textbf{Params} & \textbf{Acc@1} & \textbf{Acc@5} & \textbf{INet-O} & \textbf{INet-A} & \textbf{INet-R} \\
        \midrule
       ACMoE-Top 1 & 280M & $75.39$ & $92.56$ & $18.45$ & $7.13$ & $30.85$ \\
       ACMoE-Top 2 & 280M & $76.31$ & $93.14$ & $19.55$ & $9.42$ & $32.35$ \\
        \midrule
        NAMEx            & 280M & $76.85$ & $93.40$ & $20.11$ & $9.90$  & $32.93$ \\
        NAMEx-Full       & 280M & $77.42$ & $93.85$ & $20.69$ & $10.46$ & $33.44$ \\
        \rowcolor{blue!10}
        NAMEx-Full-Mom   & 280M & $\mathbf{78.15}$ & $\mathbf{94.23}$ & $\mathbf{21.16}$ & $\mathbf{11.02}$ & $\mathbf{33.95}$ \\
        \bottomrule
    \end{tabular}
    }}
    \vspace{-0.2in}
\end{table}

{ In \cref{tab:imagenet-acmoe-results}, across all ImageNet variants, the NAMEx-based models consistently outperform the ACMoE Top-1 and Top-2 baselines.  In particular, NAMEx-Full and NAMEx-Full-Mom set new best accuracies on both in-distribution metrics (Acc@1 and Acc@5)  and out-of-distribution benchmarks (INet-O, INet-A, INet-R). This underlines the strong generalization ability of NAMEx.  Even with the same parameter budget, NAMEx variants deliver better robustness to corruptions and distribution shifts.}
\begin{table}[t!]
\centering
\caption{Performance comparison across number of NBS solving iterations for the Linear router in NAMEx-Full config (Qwen1.5-MoE). All results to be filled; we keep the same marginal-gap setup. Throughout experiments, use 2 CCP iterations per layer for NAMEx-Full (chosen config below).}
\setlength{\tabcolsep}{6pt}

\resizebox{0.8\linewidth}{!}{
\begin{tabular}{llcccccc}
\toprule
\multirow{2}{*}{\textbf{Model}} & \multirow{2}{*}{\textbf{No. Iterations}} 
& \multicolumn{3}{c}{\textbf{Zero-Shot}} 
& \multicolumn{3}{c}{\textbf{Fine-tuned (SmolTalk)}} \\ 
\cmidrule(lr){3-5} \cmidrule(lr){6-8}
& & \textbf{MMLU} & \textbf{GSM8K} & \textbf{ARC} 
  & \textbf{MMLU} & \textbf{GSM8K} & \textbf{ARC} \\
\midrule
\multirow{5}{*}{Qwen1.5-MoE} 
& 2 (Chosen Config)  & 61.87 & 60.55 & 50.95 & 62.10 & 61.00 & 51.35 \\
& 5                   & 61.70 & 60.55 & 50.94 & 62.20 & 61.05 & 51.31 \\
& 20                  & 61.94 & 60.57 & 50.96 & 62.14 & 61.03 & 51.34 \\
& 40                  & 61.92 & 60.62 & 51.01 & 62.22 & 61.05 & 51.46 \\
& 60                  & 61.81 & 60.48 & 51.08 & 62.15 & 60.98 & 51.39 \\
\bottomrule
\end{tabular}
}
\label{tab:routing_linear_only_qwen15_moe}
\end{table}

\begin{table}[t!] 
\centering 
\caption{Performance comparison across routing strategies and models on MMLU, GSM8K, and ARC benchmarks for \qwen{} variants. Left: zero-shot results. Right: fine-tuned \namex{}\,-\,Full variants on SmolTalk.} \setlength{\tabcolsep}{6pt} 
\resizebox{\linewidth}{!}{ 
\begin{NiceTabular}{llcccccc} 
\toprule \multirow{2}{*}{\textbf{Routing Strategy}} & \multirow{2}{*}{\textbf{Model}} & \multicolumn{3}{c}{\textbf{Zero-Shot}} & \multicolumn{3}{c}{\textbf{Fine-tuned (SmolTalk)}} \\ 
\cmidrule(lr){3-5} \cmidrule(lr){6-8} & & \textbf{MMLU} & \textbf{GSM8K} & \textbf{ARC} & \textbf{MMLU} & \textbf{GSM8K} & \textbf{ARC} \\ 
\midrule \multirow{4}{*}{Linear} & \qwen{} & 61.28 & 60.12 & 50.77 & 61.50 & 60.52 & 51.12 \\ & \epcamex{} & 61.54 & 60.23 & 50.83 & 61.74 & 60.63 & 51.23 \\ 
\rowcolor{blue!10} & \textbf{\namex{}\,-\,Full} ($0$ disagreement point) & \textbf{61.87} & \textbf{60.55} & \textbf{50.95} & \textbf{62.10} & \textbf{61.00} & \textbf{51.35} \\ \RowStyle{} & \textbf{\namex{}\,-\,Full} (mean disagreement point) & \textbf{61.78} & \textbf{60.57} & \textbf{51.23} & \textbf{61.67} & \textbf{61.04} & \textbf{51.25} \\ \midrule \multirow{4}{*}{Cosine} & \qwen{} & 61.10 & 59.88 & 50.60 & 61.30 & 60.28 & 50.95 \\ & \epcamex{} & 61.40 & 60.00 & 50.68 & 61.60 & 60.50 & 51.10 \\ \rowcolor{blue!10} & \textbf{\namex{}\,-\,Full} ($0$ disagreement point) & \textbf{61.85} & \textbf{60.52} & \textbf{50.93} & \textbf{62.05} & \textbf{60.95} & \textbf{51.30} \\ \RowStyle{} & \textbf{\namex{}\,-\,Full} (mean disagreement point) & \textbf{61.86} & \textbf{60.45} & \textbf{50.77} & \textbf{62.01} & \textbf{60.81} & \textbf{51.37} \\ \midrule \multirow{4}{*}{Stable-MoE} & \qwen{} & 61.35 & 60.22 & 50.81 & 61.60 & 60.65 & 51.20 \\ & \epcamex{} & 61.60 & 60.35 & 50.89 & 61.85 & 60.80 & 51.30 \\ \rowcolor{blue!10} & \textbf{\namex{}\,-\,Full} ($0$ disagreement point) & \textbf{61.90} & \textbf{60.60} & \textbf{50.96} & \textbf{62.15} & \textbf{61.10} & \textbf{51.45} \\ \RowStyle{} & \textbf{\namex{}\,-\,Full} (mean disagreement point) & \textbf{61.88} & \textbf{60.64} & \textbf{51.03} & \textbf{62.15} & \textbf{61.11} & \textbf{51.35} \\ 
\bottomrule 
\end{NiceTabular} } 
\label{tab:qwen_routing_merge} 
\end{table}

\section{Broader Impacts}
\label{app:broader_impacts}
NAMEx proposes a principled, game-theoretic approach to expert merging in Sparse Mixture-of-Experts (SMoE) models, addressing key limitations of heuristic and curvature-based methods. By leveraging Nash Bargaining, it enables more balanced and interpretable parameter integration, particularly in settings with conflicting or specialized expert knowledge. This has direct implications for scalable deployment, as NAMEx can reduce the memory and compute footprint of large SMoE models while preserving performance. The addition of complex momentum enhances convergence stability during expert propagation, offering a robust framework for layered expert interaction. These contributions may prove valuable for future research in modular deep learning, federated optimization, and transfer learning, where efficient and fair expert combination is critical.

\end{document}